\documentclass{article}

\PassOptionsToPackage{numbers, compress}{natbib}
 \usepackage[preprint]{neurips_2026}

\usepackage{microtype}
\usepackage{graphicx}
\usepackage{subcaption}
\usepackage{booktabs} 
\usepackage{comment}
\usepackage{wrapfig}
\usepackage{tikz}
\usepackage{tabularx}
\usepackage{etoc}

\usepackage{xurl}

\usepackage{multirow}
\usepackage{makecell}
\usepackage{colortbl} 
\usepackage{array}
\usepackage{textcomp}
\usepackage[most]{tcolorbox}
\tcbuselibrary{listings,breakable}

\usepackage{listings}
\usepackage{xcolor}

\lstdefinestyle{graspPython}{
  language=Python,
  basicstyle=\ttfamily\small,
  columns=fullflexible,
  breaklines=true,
  breakatwhitespace=true,
  showstringspaces=false,
  upquote=true,
  frame=single,
  rulecolor=\color{black},
  xleftmargin=0.6em,
  xrightmargin=0.6em,
  aboveskip=0.6em,
  belowskip=0.2em,
  keywordstyle=\bfseries,
  commentstyle=\itshape,
  numbers=none,
  morestring=[s]{"""}{"""},
  morestring=[s]{'''}{'''},
  stringstyle=\ttfamily,
}

\newcommand{\BEAS}{\begin{eqnarray*}}
\newcommand{\EEAS}{\end{eqnarray*}}
\newcommand{\BEA}{\begin{eqnarray}}
\newcommand{\EEA}{\end{eqnarray}}
\newcommand{\BEQ}{\begin{equation}}
\newcommand{\EEQ}{\end{equation}}
\newcommand{\BIT}{\begin{itemize}}
\newcommand{\EIT}{\end{itemize}}
\newcommand{\BNUM}{\begin{enumerate}}
\newcommand{\ENUM}{\end{enumerate}}
\newcommand{\BA}{\begin{array}}
\newcommand{\EA}{\end{array}}

\newcommand{\set}{\mathcal{S}}
\newcommand{\R}{\mathbb{R}}

\newcommand{\N}{\mathbb{N}}
\newcommand{\NN}{\mathcal{N}}
\newcommand{\DD}{\mathcal{D}}

\usepackage{framed}
\usepackage{fvextra} 

\DefineVerbatimEnvironment{PromptVerbatim}{Verbatim}{
  fontsize=\scriptsize,
  breaklines=true,
  breakanywhere=true,
  breaksymbolleft=\raisebox{0.6ex}{\tiny$\hookrightarrow$}\,,
  frame=single,
  framesep=2.5mm
}
\usepackage{hyperref}

\usepackage{amsmath}
\usepackage{amssymb}
\usepackage{mathtools}
\usepackage{amsthm}
\usepackage{xltabular}

\usepackage[capitalize,noabbrev]{cleveref}

\theoremstyle{plain}
\newtheorem{theorem}{Theorem}[section]

\newtheorem{lemma}[theorem]{Lemma}

\theoremstyle{definition}
\newtheorem{definition}[theorem]{Definition}

\theoremstyle{remark}
\newtheorem{remark}[theorem]{Remark}

\usepackage[textsize=tiny]{todonotes}

\title{GRASP: Deterministic argument ranking in interaction graphs}

\author{%
  Diganta Misra$^{1,2,3,4}$ \quad Antonio Orvieto$^{1,2,3}$ \quad Rediet Abebe$^{1,2,3}$ \quad Volkan Cevher$^{5}$ \\
  \\
  $^1$MPI-IS Tübingen, $^2$Tübingen AI Center, $^3$ELLIS Institute Tübingen \\
  $^4$Eberhard Karls Universität Tübingen, $^5$LIONS, EPFL \\
  \texttt{diganta.misra@tue.ellis.eu} \\
}

\begin{document}

\maketitle
\begin{abstract}
Large language models are increasingly deployed as automated judges to evaluate the strength of arguments. As this role expands, their legitimacy depends on consistency, transparency, and the ability to separate argumentative structure from rhetorical appeal. However, we show that holistic judging---a common LLM-as-a-Judge practice where a model provides a global verdict on a debate---suffers from substantial inter-model disagreement. We argue that this instability arises from collapsing a debate's complex interaction structure into a single opaque score. To address this, we propose GRASP (Gradual Ranking with Attacks and Support Propagation), a deterministic framework that aggregates stable local interaction judgments into a global ranking via a convergent attack--defense propagation operator. We show that local interaction judgments are more reproducible than holistic rankings in LLM-as-a-Judge evaluations, allowing GRASP to produce more consistent global rankings. We further show that GRASP scores do not correlate with human ``convincingness'' labels, highlighting a vital sociotechnical distinction: GRASP does not measure persuasion, factuality, or rhetorical appeal, but structural sufficiency---a defense-aware notion of argument robustness over the explicit interaction graph. Overall, GRASP offers a transparent and auditable alternative to holistic LLM judging.
\end{abstract}

\section{Introduction}\label{sec:intro}

Large language models (LLMs) are increasingly used not only to generate content, but also
to act as automated judges for evaluating discourse quality, moderating debates~\cite{chuang2025debatelargescalebenchmarkroleplaying}, and
supporting multi-agent decision-making
\citep{bai2022traininghelpfulharmlessassistant,ouyang2022training,achiam2023gpt,liang2022holistic}.
For such systems to be credible arbiters of deliberation, their evaluations must be
consistent, transparent, and grounded in explicit argumentative structure rather than
model-specific idiosyncrasies.

Most current practice relies on \emph{holistic judging}, where an LLM is presented with an
entire debate and asked to output a global verdict or ranking
\citep{zheng2023judging,liang2022holistic,thakur2025judging}.
We show that these global judgments exhibit substantial inter-model disagreement, consistent
with growing concerns about the reliability of LLM-as-a-Judge paradigms
\citep{ye2024justice,chehbouni2025neither}.
Such instability suggests that holistic judging may conflate argumentative structure with
rhetorical and stylistic preferences---e.g., biases, verbosity, tone, or narrative coherence
\citep{taubenfeld2024systematic,stureborg2024large,hu2024explaining}.

We argue that this reflects a structural limitation of the paradigm: global judgments collapse
rich dialectical interactions into a single black-box score.
Instead, we ground evaluation in \emph{local semantic interactions}, drawing on computational
argumentation and abstract argumentation frameworks
\citep{dung1995acceptability,besnard2001logic,bench2007audiences,wachsmuth2017computational}.
While models may disagree on holistic verdicts, they are markedly more consistent on local
pairwise judgments
\citep{nie2020adversarial,bowman2015snli,liu2024aligning}, making such
interactions a more reliable primitive. We therefore ask:

\begin{center}
\fbox{
\begin{minipage}{0.92\linewidth}
\textbf{Research question.}
Can argument ranking be grounded in explicit local interaction structure rather than opaque holistic LLM-as-a-Judge verdicts?
\end{minipage}
}
\end{center}

To address this question, we introduce \textsc{GRASP}
(\textbf{G}radual \textbf{R}anking with \textbf{A}ttacks and \textbf{S}upport \textbf{P}ropagation),
a convergent propagation algorithm that composes local attack and support judgments into a
global ranking by aggregating direct attacks and higher-order defenses on an explicit
interaction graph.
This yields a notion of argument strength tied to the explicit interaction graph rather than
to an argument's truth, persuasiveness, rhetorical quality, or human preference.

We formalize this target as \emph{structural sufficiency}: a defense-aware notion of argument
robustness defined relative to an explicit interaction graph.
Structural sufficiency is closest to global sufficiency
\citep{cohen2001evaluating,gurcke-etal-2021-assessing}, but evaluates robustness only with
respect to the instantiated structure, rather than all counterarguments that could be
anticipated.
It is related to gradual and ranking-based argumentation semantics
\citep{baroni2011introduction,amgoud2013ranking}.

Empirically, we show that \textsc{GRASP} produces substantially more reproducible rankings
than direct LLM judging, while its weak correlation with human ``convincingness'' labels
supports the distinction between structural sufficiency and persuasive effectiveness.

\vspace{0.5em}
\noindent\textbf{Contributions.}
\begin{itemize}
\item We document substantial inter-model disagreement in holistic LLM-based judging
and show that local interaction judgments are more reproducible (\S\ref{sec:experiments}, Appendix~\ref{app:idebate_sanity}).
\item We introduce \textbf{\textsc{GRASP}}, a convergent propagation algorithm for structural
argument ranking over explicit interaction graphs (\S\ref{sec:grasp_operator}, \S\ref{sec:grasp_dynamics}).
\item We formalize \textbf{structural sufficiency}, a defense-aware notion of argument
robustness over explicit interaction graphs (\S\ref{sec:struct_suff}).
\item We show that \textsc{GRASP}-based rankings are highly reproducible across models while
capturing a notion of robustness distinct from human convincingness
(\S\ref{sec:experiments},     Appendix~\ref{sec:ddo}).
\end{itemize}

\section{GRASP: A Structural Strength Propagation Operator}
\label{sec:grasp_operator}

Motivated by the unreliability of holistic argument evaluation and the relative consistency
of local relational judgments, we introduce \textsc{GRASP} (Gradual Ranking with Attacks and
Support Propagation), an iterative operator that computes continuous argument strengths from
an explicit interaction graph. \textsc{GRASP} aggregates local attack and defense relations
into a global ranking while remaining agnostic to rhetorical appeal, persuasion, or
hypothetical objections.

\subsection{Background: Abstract Argumentation and Ranking Semantics}

We build on \emph{Abstract Argumentation Frameworks} (AAFs)~\citep{dung1995acceptability}, where arguments are nodes and
attacks are directed edges. Classical AAF semantics focus on set-valued notions of
acceptability, such as admissible, preferred, stable, and grounded extensions
~\citep{bonzon2016comparative,baroni2011introduction}, which are ill-suited for
fine-grained comparison among individual arguments.

Ranking-based semantics address this limitation by assigning each argument a numerical or
ordinal strength~\citep{amgoud2013ranking}. A prominent example is the
\emph{H-categorizer}~\citep{besnard2001logic}, which penalizes arguments according to the
total strength of their attackers. While effective as a local
heuristic, such methods treat attacks independently and do not explicitly model how
arguments may be defended by other arguments.

Our aim is to retain the simplicity and interpretability of attack-based rankings while
incorporating defense through an explicit propagation rule.

\subsection{Weighted Interaction Graphs}

Let $A = \{a_1,\dots,a_n\}$ be a set of arguments. \textsc{GRASP} operates on a weighted
interaction graph $\mathcal{G}=(A,W,D)$, with attack matrix
$W \in [0,1]^{n \times n}$ and defense matrix $D \in [0,\infty)^{n \times n}$.
Here, $W_{ij}$ is the strength with which $a_i$ attacks $a_j$, and $D_{kj}$ is the extent
to which $a_k$ defends $a_j$.

\subsection{The GRASP Update Rule}

Given an initial strength vector $\mathbf{s}^{(0)} \in \mathbb{R}_{\ge 0}^n$, \textsc{GRASP}
iteratively updates argument strengths through a nonlinear operator that balances the
weakening effect of attacks against the reinforcing effect of defense.

\paragraph{Undamped operator.}
We first define the undamped GRASP operator $G : \mathbb{R}^n \to \mathbb{R}^n$
coordinatewise
\begin{equation}
\label{eq:grasp_core}
G(s)_j
\;=\;
\frac{1 + \beta \sum_k D_{kj} s_k}
     {1 + \alpha \sum_i W_{ij} s_i},
\end{equation}
where the denominator penalizes argument $a_j$ according to the total strength of its
attackers, while the numerator rewards it according to the total strength of its defenders.
The parameters $\alpha,\beta \ge 0$ control the relative influence of attack and defense.

\begin{wrapfigure}{r}{0.55\textwidth}
\vspace{-8mm}
\phantomsection
\label{def:grasp}
\begin{tcolorbox}[
    title=GRASP Update Rule, 
    colback=blue!5!white, 
    colframe=blue!75!black, 
    boxsep=2pt,
    left=2pt,
    right=2pt
]
\footnotesize
Define the damped operator $\widehat G : \mathbb{R}^n \to \mathbb{R}^n$:
\begin{equation}
\label{eq:grasp-operator}
\widehat G(s) = (1-\gamma)s + \gamma\, G(s)
\end{equation}
The iteration is then given by:
\begin{equation}
\label{eq:damped-update}
s^{(t)} = \widehat G\bigl(s^{(t-1)}\bigr)
\end{equation}
or coordinatewise:
\begin{equation}
\label{eq:damped-update-coord}
s_j^{(t)} = (1 - \gamma) s_j^{(t-1)}
+ \gamma \frac{1 + \beta \sum_{k} D_{kj} s_k^{(t-1)}}
     {1 + \alpha \sum_{i} W_{ij} s_i^{(t-1)}}
\end{equation}
\end{tcolorbox}
\vspace{-9mm}
\end{wrapfigure}

\paragraph{Damped GRASP operator and iteration.}
To improve numerical stability and guarantee convergence in dense or highly cyclic
interaction structures, we employ a damped version of the operator, analogous to relaxation
schemes in iterative optimization. Repeated application yields stable argument strengths reflecting how well each argument withstands attack within the explicit interaction structure. Final rankings order arguments by converged strength.

\subsection{Interpretation}

\textsc{GRASP} generalizes attack-based ranking methods such as the H-categorizer by
explicitly incorporating defense while preserving locality and interpretability. The
resulting scores operationalize the notion of \emph{structural sufficiency} introduced in
Section~\ref{sec:struct_suff}: arguments are strong to the extent that incoming attacks are offset by available defense
in the instantiated structure.

The GRASP operator is independent of any specific graph-construction procedure. Different
choices of $W$ and $D$ encode different assumptions about what counts as attack and defense;
once these matrices are specified, \textsc{GRASP} provides a deterministic structural
aggregation rule.

\section{Convergence of the GRASP Operator}
\label{sec:grasp_dynamics}

We now analyze the dynamics induced by the \textsc{GRASP} update rule from
Section~\ref{sec:grasp_operator}. The iteration defines a nonlinear operator on argument
strength vectors; its stability, fixed points, and convergence are essential for interpreting
the resulting structural rankings.

In  Appendix~\ref{sec:proof}, we study the GRASP update as a map $G:\mathbb{R}^n\to \mathbb{R}^n$, induced by the non-symmetric weighted attack matrix $W$ and the derived defense matrix $D$. 
The nonlinearity of the operator arises from the interaction between attack aggregation in the denominator and defense propagation in the numerator, reflecting the dialectical coupling between opposition and reinstatement.



\begin{theorem}\label{thm:grasp} Let $\set:=\{s\in\R^d, \|s-1\|_\infty\le 1\}$ and let $G:\set\to\R^d$ be defined elementwise by $ G(s)_i = \frac{1 + \beta (D^\top s)_i}{1 + \alpha(W^\top s)_i}$. If $W,D$ have non-negative entries and
    \[
    \alpha\le\frac{1}{4\|W\|_1},
    \qquad
    \beta\le\frac{1}{4\|D\|_1},
    \]
    with $\|A\|_1 := \max_j \sum_i |a_{ij}|$, then $G(\set)\subseteq \set$ and $G$ is a contraction on $\set$. Consequently, $G$ admits a unique fixed point $s^*\in\set$, and the iteration $s_{k+1}=G(s_k)$ converges to $s^*$. The result extends to the damped variant in Eq.~\ref{eq:grasp-operator}.
\end{theorem}

 Our proof in Appendix~\ref{sec:proof} uses standard tools but applies them to a non-standard and non-linear setup. Given the worst-case nature of our analysis, we note that the coefficients $\alpha$ and $ \beta$ suggested by the proof might not yield the best-performing results in general, and hence treat them as tuning parameters in our experiments~(see e.g. Section~\ref{sec:agreement}).

\section{Structural Sufficiency}
\label{sec:struct_suff}

The GRASP operator introduced in the previous section computes a stable equilibrium of
argument strengths from explicit interaction relations. We now clarify the notion of
strength that this equilibrium represents. In our setting, strength does not refer to truth,
persuasiveness, rhetorical quality, or human preference, but to structural robustness: the
extent to which an argument withstands the explicit attacks present in the instantiated
interaction graph.

Within the argumentation literature, the closest conceptual analogue is \emph{global
sufficiency}~\cite{gurcke-etal-2021-assessing,cohen2001evaluating,damer1980attacking}: the idea that an
argument is strong if it adequately withstands opposing arguments. However, global sufficiency
is typically defined relative to attacks that \emph{could be anticipated}, including
hypothetical or implicit counterarguments. In contrast, GRASP operates strictly on the
instantiated interaction structure. This motivates the following graph-relative variant.

\vspace{-2mm}
\paragraph{Structural sufficiency.}
We define \emph{structural sufficiency} as a dialectical notion of argument robustness that
depends only on the \emph{explicit} interaction structure present in a debate. It abstracts
away from truth, rhetorical appeal, persuasion, human preference, and imagined or hypothetical attacks.

Table~\ref{tab:sufficiency-comparison} summarizes the distinction between global sufficiency
and structural sufficiency.

\begin{table}[h]
\centering
\vspace{-3mm}
\small
\setlength{\tabcolsep}{5pt}
\renewcommand{\arraystretch}{1.08}
\begin{tabular}{
p{0.18\linewidth}
p{0.37\linewidth}
>{\columncolor{gray!10}}p{0.37\linewidth}
}
\toprule
 & \textbf{Global sufficiency} & \textbf{Structural sufficiency} \\
\midrule
\textbf{Criterion type}
& Dialectical
& Dialectical \\
\textbf{Scope}
& Actual, anticipated, or implicit attacks
& Explicit attacks in the graph \\
\textbf{Target}
& Robustness in broader discourse
& Robustness in instantiated structure \\
\textbf{Excludes}
& Not necessarily separated from persuasion or context
& Truth, persuasion, style, imagined attacks \\
\bottomrule
\end{tabular}
\caption{Structural sufficiency restricts dialectical evaluation to explicit attacks in the interaction graph, whereas global sufficiency may require addressing anticipated or implicit attacks.}
\vspace{-4mm}
\label{tab:sufficiency-comparison}
\end{table}

We proceed with a few definitions formalizing this notion.

\vspace{-2mm}
\paragraph{Structure.}
An \emph{argumentation structure} is a tuple $\mathcal{G}=(A,R^-,R^+)$, where $A$ is a
finite set of arguments, $R^- \subseteq A\times A$ is an attack relation, and
$R^+ \subseteq A\times A$ is a support relation. We write $(b,a)\in R^-$ as
``$b$ attacks $a$'' and $(c,a)\in R^+$ as ``$c$ supports $a$.''

\vspace{-2mm}
\paragraph{Neutralization.}
Fix an argument $a\in A$ and an attacker $b\in A$ such that $(b,a)\in R^-$. We say that
the attack $(b,a)$ is \emph{structurally neutralized} in $\mathcal{G}$ if there exists at
least one argument $c\in A$ such that $(c,b)\in R^-$. That is, an attack is neutralized
whenever the attacker itself is explicitly attacked within the structure.

Support relations may enable potential defenders, but structural sufficiency does not require
any specific operational treatment of support edges.

\vspace{-2mm}
\paragraph{Structural sufficiency.}
An argument $a$ is \emph{structurally sufficient} in $\mathcal{G}$ if every explicit attack
on $a$ is neutralized:
\[
\mathrm{SS}(a;\mathcal{G})
\;\Longleftrightarrow\;
\forall b \in A,\; (b,a) \in R^{-}
\Rightarrow \exists c \in A \text{ such that } (c,b) \in R^{-}.
\]
All quantification ranges only over arguments explicitly present in $A$. Thus, unlike
\emph{global sufficiency}, structural sufficiency evaluates robustness strictly with respect
to the instantiated interaction structure~\cite{wachsmuth2017computational}.

\vspace{-2mm}

\paragraph{Axioms of Structural Sufficiency.}
The definitions above imply four minimal desiderata for any interaction-based robustness
criterion based on structural sufficiency:

\vspace{-1mm}
\begin{center}
\small
\setlength{\tabcolsep}{5pt}
\renewcommand{\arraystretch}{1.18}
\begin{tabular}{p{0.30\linewidth}p{0.66\linewidth}}
\toprule
\textbf{Axiom} & \textbf{Statement} \\
\midrule
\textbf{S1: Attack Sensitivity}
& Unneutralized attacks invalidate sufficiency. \newline
\((b,a)\in R^- \land \neg\exists c\in A : (c,b)\in R^-
\Rightarrow \neg \mathrm{SS}(a;\mathcal{G})\). \\

\textbf{S2: Defense Reinstatement}
& Attacking an attacker restores sufficiency with respect to that attack. \newline
\((c,b)\in R^- \land (b,a)\in R^-
\Rightarrow (b,a)\text{ is neutralized}\). \\

\textbf{S3: Structural Locality}
& Only structurally connected arguments affect sufficiency. \newline
\(\neg\exists\) directed path \(x\leadsto a\) in \((A,R^-\cup R^+)
\Rightarrow x\) has no effect on \(\mathrm{SS}(a;\mathcal{G})\). \\

\textbf{S4: Baseline Sufficiency}
& Arguments without attackers are sufficient by default. \newline
\(\neg\exists b\in A \;\text{s.t. } (b,a)\in R^-
\Rightarrow \mathrm{SS}(a;\mathcal{G})\). \\
\bottomrule
\end{tabular}
\end{center}
\vspace{-1mm}

Together, these axioms define a Boolean robustness criterion: an argument is sufficient iff all its explicit attackers are countered in the graph. While conceptually simple, this
criterion does not support graded comparison. \textsc{GRASP} uses the same primitives---explicit attacks, defense, and structural
locality---but aggregates them through weighted propagation and normalization. In this sense, \textsc{GRASP} is a graded, propagation-based analogue of the structural sufficiency intuition, yielding continuous strength scores rather than binary structural sufficiency labels.

\section{Experiments}
\label{sec:experiments}

\subsection{\textsc{StructDebate}}
\label{sec:structdebate}

We introduce \textsc{StructDebate}, a controlled debate dataset designed to study
structural argument ranking under explicitly instantiated interaction regimes. Its purpose is not to model persuasion, but to isolate attack--defense structure from
rhetorical and stylistic variation in a reproducible setting.

\begin{wraptable}{r}{0.50\textwidth}
\centering
\small
\vspace{-4mm}
\setlength{\tabcolsep}{4pt}
\resizebox{\linewidth}{!}{
\begin{tabular}{lcc}
\toprule
 & \textbf{Pool} & \textbf{Multi-turn} \\
\midrule
\# Debates                & 50      & 250 \\
\# Arguments              & 2{,}000 & 5{,}000 \\
Arguments / debate        & 40.0    & 20.0 \\
Mean length (tokens)      & 42.5    & 94.9 \\
Turns / debate            & --      & 10 \\
\bottomrule
\end{tabular}
}
\caption{Summary statistics for \textsc{StructDebate}.}
\label{tab:dataset_stats}
\end{wraptable}

\textsc{StructDebate} contains machine-generated arguments grounded in 50 real-world motions
sampled from the public DebateData.io corpus.\footnote{\url{https://debatedata.io/}}
The motions span public policy, economics, technology, law, and ethics, and are phrased as
binary propositions (Pro/Con) suitable for adversarial debate. Arguments are generated using five
LLMs---\texttt{openai/gpt-5.2-pro}~\citep{openai2025gpt52systemcard},
\texttt{anthropic/claude-opus-4.5}~\citep{anthropic2025claudeopus45systemcard},
\texttt{mistralai/mistral-small-creative}~\citep{mistral2025smallcreative},
\texttt{qwen/qwen3-max}~\citep{qwen2025qwen3max}, and
\texttt{x-ai/grok-4}~\citep{xai2025grok4modelcard}---which are used only as generators,
not as judges.

We generate arguments in two settings. In the \emph{pool} setting (P), arguments are generated
independently for each motion, side, and semantic angle. In the \emph{multi-turn} setting (MT),
arguments are generated sequentially in 10-turn self-debates with alternating
\textsc{Pro}/\textsc{Con} turns. Each argument is associated with one of six semantic
angles: \textsc{Economic}, \textsc{Legal}, \textsc{Moral}, \textsc{Political},
\textsc{Social}, and \textsc{Technological}. The rationale for the six semantic angles and the full list of 50 motions are provided in Appendix~\ref{app:dataset-details}.

Overall, \textsc{StructDebate} contains 7{,}000 arguments across 300 debates and is balanced
by stance and semantic angle.

\subsection{Inter-Model Agreement via Structural Aggregation}
\label{sec:agreement}

A central motivation for structural evaluation is that judge models often produce
inconsistent argument rankings when operating directly on raw text. We test whether replacing
holistic ranking with pairwise interaction scoring followed by deterministic structural
aggregation yields more stable rankings across models.

\vspace{-3mm}

\paragraph{Constructing the attack graph.}
For each debate, we instantiate a fully connected directed graph over arguments.
Inspired by natural language inference (NLI), each judge LLM is prompted to score every
ordered pair $(a_i,a_j)$ by how strongly argument $a_i$ contradicts or attacks argument
$a_j$.
The resulting score is treated as a weighted attack from $a_i$ to $a_j$, yielding a dense
attack matrix $W$.

Each judge model independently produces its own $W$ using the same prompting template.
\textsc{GRASP} then converts each $W$ into a global ranking by computing argument strengths
from the induced interaction graph and sorting arguments by their final scores.

Appendix~\ref{app:idebate_sanity} provides an external sanity check showing that this pairwise scoring prompt recovers human-authored point--counterpoint relations from iDebate/IDEA pages.

\vspace{-3mm}

\paragraph{Baselines and judge models.}
\label{sec:baseline_models}
We compare against \textsc{RAW}, where a judge model directly ranks the arguments from the
full debate text. To test whether direct LLM judging improves when given explicit structural guidance, we include \textsc{RAW+SS}, a strengthened baseline that provides the judge with a definition of structural sufficiency before asking for a ranking. This tests whether the instability of \textsc{RAW} rankings can be mitigated by clarifying the target criterion alone. Exact prompt details are provided in Appendix~\ref{prompt:raw_ss}.

We use six judge models: \texttt{anthropic/claude-haiku-4.5}~\citep{anthropic2025claudehaiku45},
\texttt{deepseek/deepseek-v3.2}~\citep{liu2025deepseek},
\texttt{google/gemini-3-flash-preview}~\citep{google2025gemini3flash},
\texttt{meta-llama/llama-4-scout}~\citep{meta2025llama4},
\texttt{openai/gpt-5.2-chat}~\citep{openai2025gpt52systemcard}, and
\texttt{xiaomi/mimo-v2-flash}~\citep{xiao2026mimo}. In \textsc{GRASP}, judge models only score local attack relations; the final ranking is computed by structural aggregation.
\vspace{-7mm}

\paragraph{GRASP variants.}
For the main experiments, we instantiate defense as \(D=W^2=W.W\). This is not required by the
GRASP operator; it is the default graph-derived choice for attack-only debate graphs. A
two-hop path \(a_k \rightarrow a_i \rightarrow a_j\) means that \(a_k\) attacks an attacker
of \(a_j\), and therefore contributes to defending \(a_j\). We evaluate alternative choices
of \(D\) in the synthetic structural testbed in Appendix~\ref{sec:synthetic_structural}.

We evaluate five variants that differ only in preprocessing. \textbf{GRASP} uses \(W\) as produced by the judge and \(D=W^2\).
\textbf{GRASP-W\(_\infty\)} and \textbf{GRASP-W\(_1\)} apply global \(L_\infty\) or \(L_1\)
normalization to \(W\) before computing \(D\). The \(+\bar D\) variants additionally
rescale the induced defense matrix. All variants use the same update rule.

All runs initialize strengths uniformly with \(s^{(0)}=\mathbf{1}\). Unless otherwise
stated, hyperparameters are fixed \emph{a priori} to \(\alpha=1.0\), \(\beta=0.6\), and
damping \(\gamma=0.9\); sensitivity analysis over \((\alpha,\beta,\gamma)\) is reported in
Appendix~\ref{app:gridsearch}.

\vspace{-3mm}

\paragraph{Agreement metrics.}
Within each debate, we compute pairwise agreement between judge rankings using Kendall's
$\tau$, Spearman's $\rho$, normalized Kendall swap distance, and Top-$k$ overlap, and
average the resulting scores across debates.
\vspace{-3mm}

\paragraph{Results.}
Table~\ref{tab:agreement_settingwise} shows that GRASP-based structural aggregation more
than doubles inter-model agreement relative to RAW rankings in both pool and multi-turn
settings. Providing judges with the definition of structural sufficiency alone
(\textsc{RAW+SS}) does not resolve this instability: it slightly improves Pool agreement but substantially degrades Multi-turn agreement and top-\(k\) overlap. In contrast, the base GRASP and GRASP-W$_\infty$ variants achieve the strongest and most stable performance ($\tau \approx 0.62$), indicating that the improvement comes from explicit structural aggregation rather than from simply clarifying the evaluation criterion.

\begin{table}[t]
\centering
\small
\setlength{\tabcolsep}{5pt} 
\aboverulesep=0ex
\belowrulesep=0ex
\resizebox{\columnwidth}{!}{%
\begin{tabular}{l cc cc cc cc cc}
\toprule
& \multicolumn{2}{c}{$\tau$ {\footnotesize ($+\Delta\tau$)} $\uparrow$} & \multicolumn{2}{c}{Swap $\downarrow$} & \multicolumn{2}{c}{$\rho \uparrow$} & \multicolumn{2}{c}{Top-3 $\uparrow$} & \multicolumn{2}{c}{Top-5 $\uparrow$} \\
\cmidrule(lr){2-3} \cmidrule(lr){4-5} \cmidrule(lr){6-7} \cmidrule(lr){8-9} \cmidrule(lr){10-11}
Method & P & MT & P & MT & P & MT & P & MT & P & MT \\
\midrule
RAW                       & 0.337 & 0.309 & 0.331 & 0.345 & 0.425 & 0.380 & 0.385 & 0.410 & 0.417 & 0.474 \\
RAW+SS & 0.393 & 0.064 & 0.303 & 0.468 & 0.405 & 0.038 & 0.208 & 0.251 & 0.235 & 0.298 \\
\rowcolor{gray!15}
\textbf{GRASP}     & \textbf{0.623} {\scriptsize (+.286)} & \textbf{0.626} {\scriptsize (+.317)} & 0.189 & \textbf{0.187} & 0.780 & \textbf{0.779} & 0.509 & 0.487 & 0.574 & 0.619 \\
GRASP-W$_\infty$          & \textbf{0.623} {\scriptsize (+.286)} & \textbf{0.626} {\scriptsize (+.317)} & \textbf{0.188} & \textbf{0.187} & \textbf{0.781} & \textbf{0.779} & 0.509 & 0.486 & 0.574 & 0.620 \\
GRASP-W$_1$               & 0.604 {\scriptsize (+.267)} & 0.607 {\scriptsize (+.298)} & 0.198 & 0.197 & 0.772 & 0.768 & 0.528 & 0.525 & 0.587 & 0.634 \\
GRASP-W$_\infty$+$\bar{D}$ & 0.591 {\scriptsize (+.254)} & 0.593 {\scriptsize (+.284)} & 0.205 & 0.204 & 0.761 & 0.755 & 0.530 & 0.528 & 0.584 & 0.632 \\
GRASP-W$_1$+$\bar{D}$      & 0.580 {\scriptsize (+.243)} & 0.575 {\scriptsize (+.266)} & 0.210 & 0.212 & 0.752 & 0.740 & \textbf{0.535} & \textbf{0.542} & \textbf{0.593} & \textbf{0.635} \\
\bottomrule
\end{tabular}
}
\vspace{2mm}
\caption{Setting-wise inter-model agreement for Pool (P) and Multi-turn (MT). Improvements (\(\Delta\tau\)) over RAW are shown in parentheses. \textbf{Bold} indicates best performance; \colorbox{gray!20}{gray} is the default GRASP variant.}
\vspace{-9mm}
\label{tab:agreement_settingwise}
\end{table}

\subsection{Consensus Divergence}
\label{sec:consensus_divergence}

\begin{wraptable}{r}{8.2cm}
\centering
\small
\setlength{\tabcolsep}{3pt}
\vspace{-4mm}
\aboverulesep=0ex
\belowrulesep=0ex
\resizebox{8cm}{!}{%
\begin{tabular}{l cc cc l}
\toprule
\multirow{2}{*}{Method} & \multicolumn{2}{c}{Borda $\downarrow$} & \multicolumn{2}{c}{Kemeny $\downarrow$} & \multirow{2}{*}{Outlier (P)} \\
\cmidrule(lr){2-3} \cmidrule(lr){4-5}
 & P & MT & P & MT & \\
\midrule
RAW                       & 204.2          & 48.7          & 204.8          & 49.2          & G-5.2 Chat \\
RAW+SS & 3597.2 & 90.1 & 3574.2 & 92.0 & DS-v3.2 \\
\rowcolor{gray!15}
\textbf{GRASP}            & \textbf{116.3} & \textbf{27.2} & \textbf{116.5} & \textbf{27.3} & L-4 Scout \\
GRASP-W$_\infty$          & 116.4          & \textbf{27.2} & 116.6          & \textbf{27.3} & L-4 Scout \\
GRASP-W$_1$               & 123.8          & 28.6          & 124.0          & 28.8          & L-4 Scout \\
GRASP-W$_\infty$+$\bar{D}$ & 129.1          & 29.4          & 128.9          & 29.5          & L-4 Scout \\
GRASP-W$_1$+$\bar{D}$      & 133.0          & 30.8          & 132.9          & 30.9          & L-4 Scout \\
\bottomrule
\end{tabular}
}
\caption{\textbf{Consensus divergence.} Largest mean Kendall distance from any judge to the
consensus ranking. Lower is better. Gray shading marks the default GRASP variant. P and MT
denote Pool and Multi-turn settings. Outliers \textit{G-5.2 Chat}, \textit{DS-v3.2} and \textit{L-4 Scout}
refer to \texttt{openai/gpt-5.2-chat}, \texttt{deepseek/deepseek-v3.2} and \texttt{meta-llama/llama-4-scout}.}
\vspace{-8mm}
\label{tab:consensus_divergence}
\end{wraptable}

Beyond pairwise agreement, we study how far individual judge models deviate from a global
consensus ranking. For each debate and aggregation method, we form two consensus rankings
across judges: (i) \emph{Borda} aggregation and (ii) a greedy approximation to the
\emph{Kemeny}-optimal ranking. For each judge model, we compute its Kendall swap distance to
the corresponding consensus ranking. We then report, for each setting and method, the judge
model with the largest mean distance to consensus.

\paragraph{Results.}
Table~\ref{tab:consensus_divergence} shows that all GRASP variants substantially reduce
worst-case divergence from consensus compared to RAW rankings. The contrast is even sharper
against \textsc{RAW+SS}: although the judge is explicitly told to rank by structural
sufficiency, its farthest-from-consensus model exhibits substantially larger divergence,
especially in the pool setting. Across both pool and multi-turn debates, the most divergent
GRASP judge is roughly \(2\times\) closer to consensus than the most divergent RAW judge and
far closer than the most divergent \textsc{RAW+SS} judge. This suggests that agreement at the level of local pairwise interactions translates not only into higher average agreement, but also into a more robust global ranking: even the farthest-from consensus judge remains substantially closer to the aggregate ranking under GRASP.

\subsection{Geometry of Attack Graphs and GRASP Dynamics}
\label{sec:graph_geometry}

\begin{wrapfigure}{r}{0.60\textwidth}
\centering
\vspace{-4mm}
\includegraphics[width=0.58\textwidth]{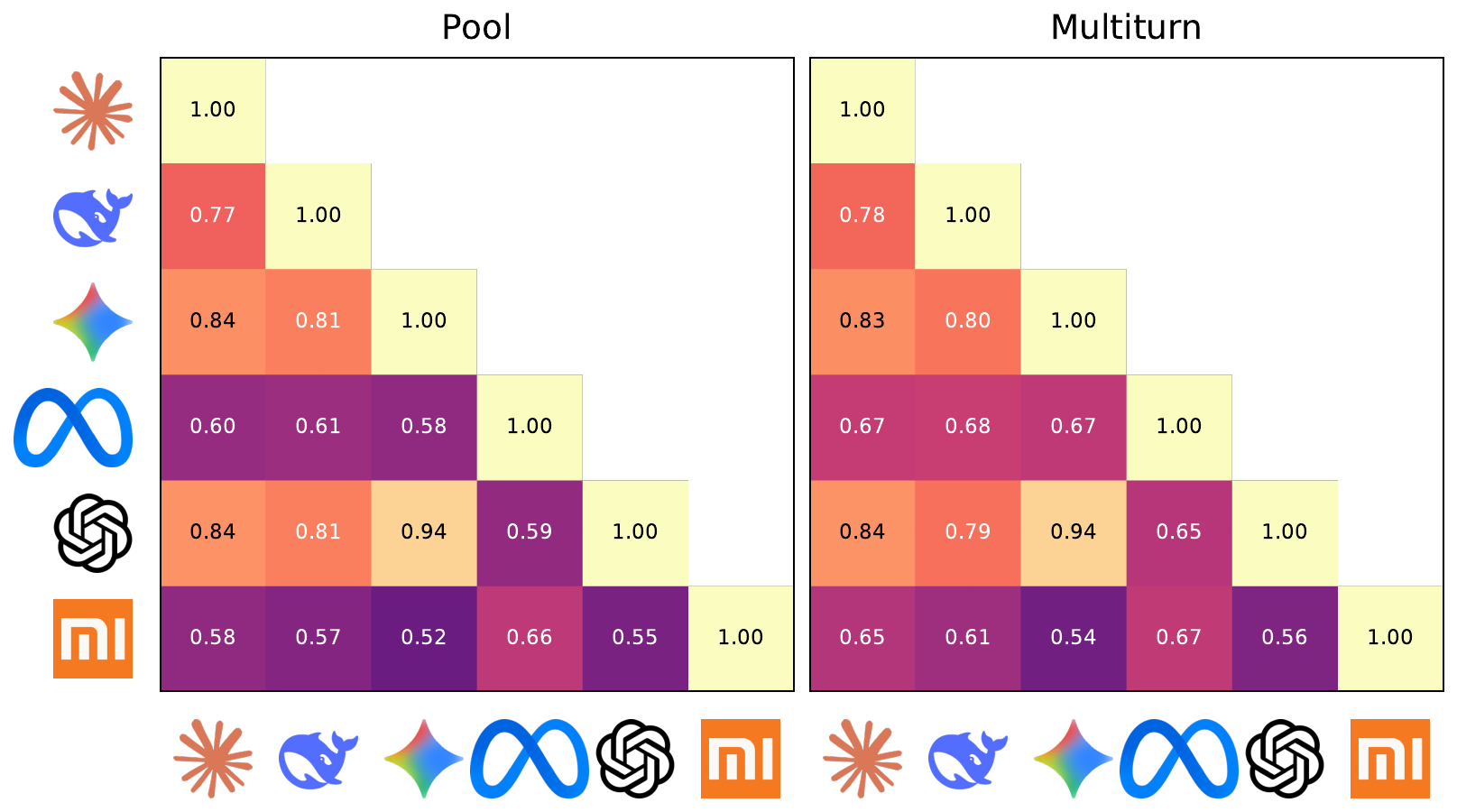}
\caption{
Pairwise mean Pearson correlation between attack-weight matrices \(W\) for Pool
(left) and Multi-turn (right) settings.
}
\label{fig:graph_similarity_heatmap}
\vspace{-8mm}
\end{wrapfigure}

A central premise of GRASP is that local interaction judgments provide a more stable
substrate than holistic rankings. Figure~\ref{fig:graph_similarity_heatmap} supports this
premise: different judge models induce strongly correlated attack-weight matrices \(W\), with
most pairwise Pearson correlations lying between \(0.55\) and \(0.95\) across both pool and
multi-turn settings. The qualitative structure of the similarity matrix is consistent across
settings, indicating that judges largely agree on the relative pattern of adversarial
relations even when they differ in absolute edge magnitudes.

\begin{figure*}[h!]
\centering
\includegraphics[width=\textwidth]{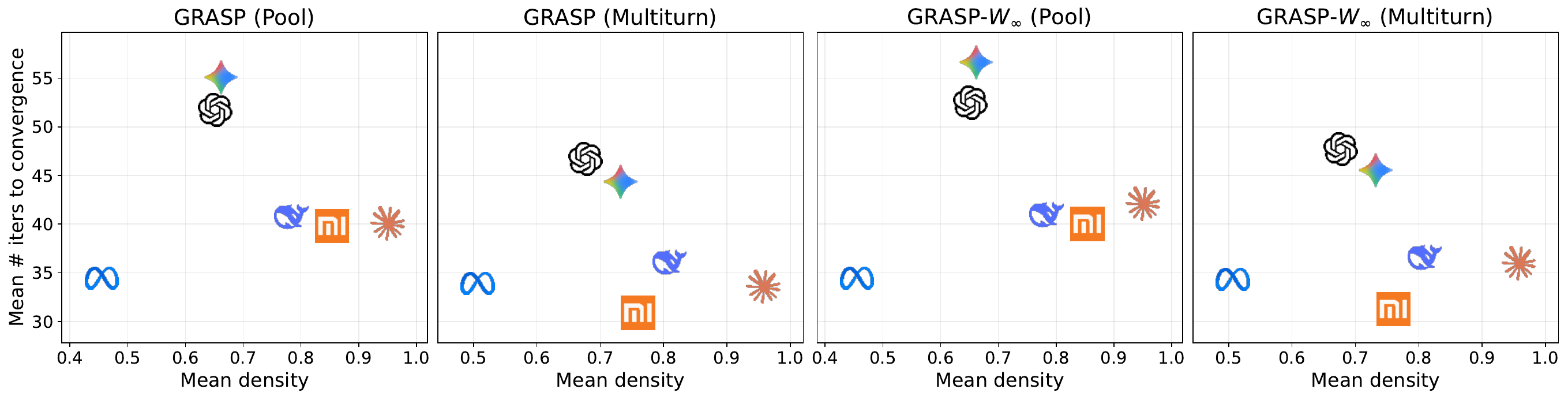}
\caption{Convergence vs.\ attack-graph density for GRASP and GRASP-\(W_\infty\).
Each logo denotes one judge model, averaged across debates.}
\label{fig:convergence}
\vspace{-3mm}
\end{figure*}

Figure~\ref{fig:convergence} relates this interaction geometry to GRASP dynamics. Although
convergence is not a simple function of graph density, judge models occupy persistent regions
of the density--iteration plane. These relative placements are preserved between GRASP and
GRASP-\(W_\infty\), suggesting that global normalization changes scale without materially
altering the underlying interaction regime.

Together, Figures~\ref{fig:graph_similarity_heatmap} and~\ref{fig:convergence} help explain
why structural aggregation improves agreement: GRASP aggregates a local interaction geometry that is substantially more reproducible than holistic rankings, and its dynamics remain stable under simple normalization.

\vspace{-2mm}

\subsection{Case Study: Graph Structure and Rank Dynamics in a Single Debate}
\label{sec:case_study}

We present a qualitative case study on \texttt{mt\_048\_x-ai\_\_grok-4}, the 48th
multi-turn debate in \textsc{StructDebate}, whose arguments were generated by
\texttt{x-ai/grok-4}. The motion is:
\emph{``This House would break up dominant technology monopolies.''}
Whereas Section~\ref{sec:graph_geometry} analyzes aggregate similarity of dense attack
matrices across debates, this case study examines a single debate after high-confidence
thresholding. The goal is to illustrate two complementary points: (i) dense attack-weight
matrices can be broadly similar while their thresholded high-confidence graphs differ, and
(ii) GRASP's iterative updates translate such graph structure into non-trivial rank dynamics.
The texts of the referenced arguments are provided in Appendix~\ref{app:case_texts}.

\paragraph{Thresholded graph structure across judge models.}
Figure~\ref{fig:same_graph} visualizes high-confidence attack graphs
(\(W_{ij}>\tau\), \(\tau=0.6\)) for this debate under six judge models. While all graphs
are constructed from the same argument set, their high-confidence structures differ in both
connectivity and organization. Some models yield dense high-confidence graphs
(\(d \approx 0.95\)--\(0.98\)), whereas others produce sparser graphs
(\(d \approx 0.5\)--\(0.65\)). The mean off-diagonal attack strength \(\mu\) also varies
substantially, indicating that models differ not only in how many attacks exceed the
threshold, but also in the overall strength they assign to pairwise attacks.

These differences arise after thresholding. Thus, the high correlations observed in the dense
attack matrices in Figure~\ref{fig:graph_similarity_heatmap} do not imply identical
high-confidence graph structure: models may agree on the relative pattern of attacks while
differing in which edges they treat as sufficiently strong.

\begin{figure*}[t]
\centering
\includegraphics[width=\linewidth]{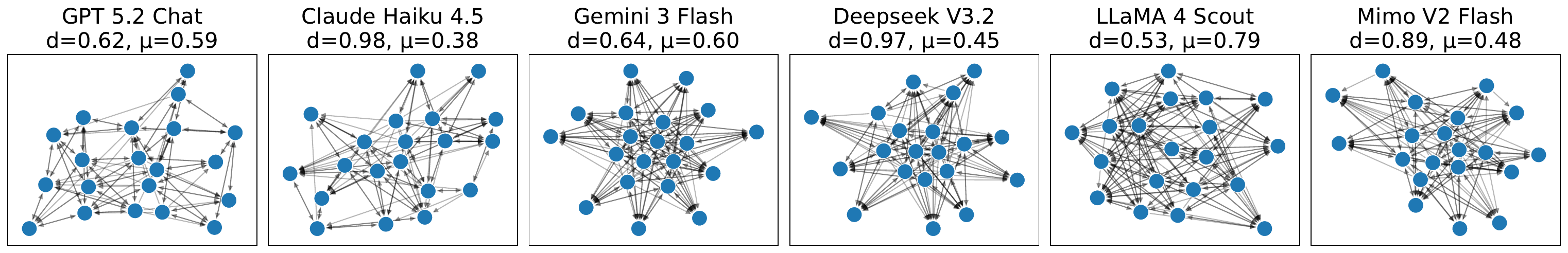}
\caption{Attack graphs for the same debate under different judge models
(\(W_{ij}>\tau\), \(\tau=0.6\)). The visualization displays only thresholded
off-diagonal edges. Each subplot reports the full positive off-diagonal density
\(d = \frac{1}{n(n-1)}\sum_{i\neq j}\mathbb{I}[W_{ij}>0]\), and the mean positive
off-diagonal attack strength
\(\mu = \frac{\sum_{i\neq j} W_{ij}\mathbb{I}[W_{ij}>0]}
{\sum_{i\neq j}\mathbb{I}[W_{ij}>0]}\), both computed before thresholding.}
\vspace{-6mm}
\label{fig:same_graph}
\end{figure*}

\paragraph{Rank dynamics under GRASP.}

\begin{wrapfigure}{r}{0.6\textwidth}
\centering
\vspace{-3mm}
\includegraphics[width=0.58\textwidth]{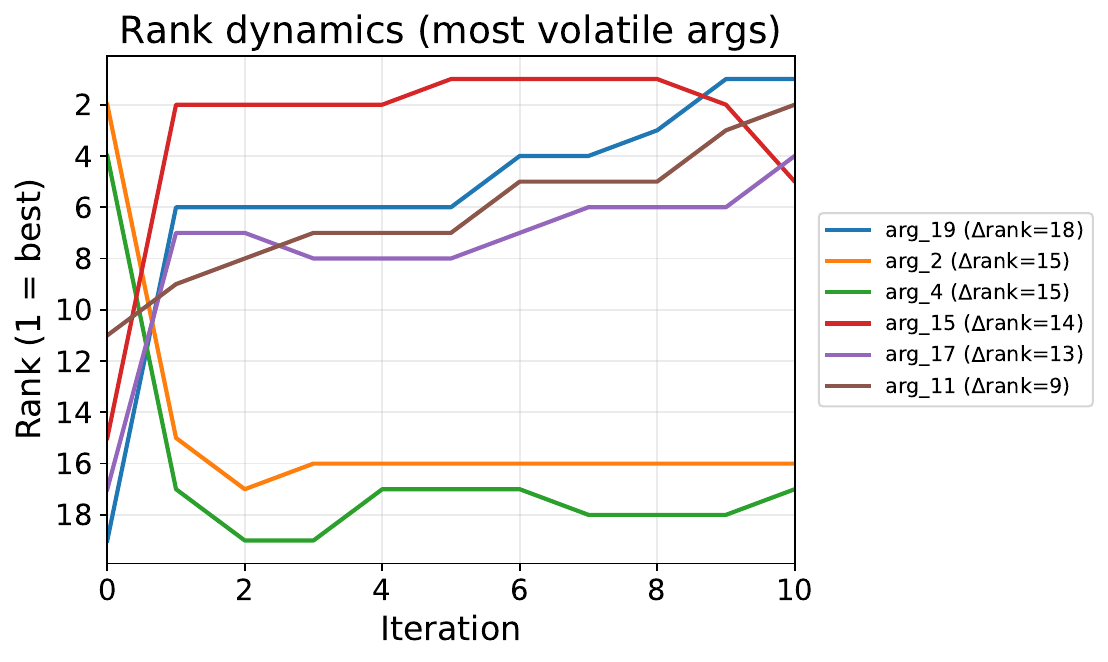}
\caption{Rank dynamics of volatile arguments for \texttt{mt\_048\_x-ai\_\_grok-4}
under GRASP using the attack graph induced by \texttt{openai/gpt-5.2-chat}.}
\label{fig:rank_dynamics}
\vspace{-3mm}
\end{wrapfigure}
Using the attack graph produced by \texttt{openai/gpt-5.2-chat}, we track GRASP scores over
iterations and visualize the rank trajectories of the most volatile arguments
(Figure~\ref{fig:rank_dynamics}, \(\tau=0.5\) for graph construction). Several arguments
undergo large early rank shifts (\(\Delta\)rank \(\approx 10\)--\(18\)) before stabilizing.

The trajectories are not monotonic: some arguments temporarily rise before being demoted,
while others steadily improve. This reflects the coupled nature of the GRASP update, where an
argument's score depends on incoming attacks, available defense, and higher-order interactions
propagated through the graph. Convergence therefore arises from coordinated global
reweighting rather than from a simple local sorting heuristic.

\vspace{-2mm}
\paragraph{Takeaway.}
Dense attack-weight matrices can be broadly correlated across judge models while their
thresholded high-confidence graphs remain visibly different. GRASP operates on these induced structures and produces interpretable rank dynamics,
indicating that its rankings are driven by explicit interactions rather than smoothed
holistic rankings.

\subsection{Case Study: Structural Consensus vs. RAW Disagreement}
\label{sec:case_study_contrast}

We present a representative case in which an argument is ranked near-unanimously near the
top by \textsc{GRASP}, yet receives dispersed and often low rankings under direct
\textsc{RAW} judging. This contrast illustrates how aggregating local pairwise interactions
into an explicit structure can stabilize global priorities, even when individual
judges disagree substantially at the level of holistic evaluation.

\begin{tcolorbox}[
    colback=green!5,
    colframe=green!60!black,
    title=\textbf{This House would require warrants for searches instead of allowing stop-and-frisk.},
    fonttitle=\bfseries,
    after skip=2mm,
    bottom=1mm
]
\small
\begin{tabular*}{\linewidth}{@{\extracolsep{\fill}}ll ll ll@{}}
\textbf{Generator:} & \texttt{openai/gpt-5.2-pro} & \textbf{Stance:} & Pro & \textbf{Angle:} & Pol. \\
\textbf{Turn:} & 4 & \textbf{\# Attk:} & 10 & \textbf{Mean:} & 0.422
\end{tabular*}
\vspace{1mm} \hrule \vspace{2mm}
\textbf{Argument:} \emph{Politically, requiring warrants shifts the authority to search from unilateral street-level discretion to a process that includes independent oversight, reinforcing separation of powers and democratic control over coercive state action. This reduces the risk that search practices become informal policy tools shaped by electoral pressures or internal quotas rather than publicly accountable standards. Clear warrant rules also create more consistent statewide governance, limiting local variations that can undermine legitimacy and deepen political polarization over policing.}
\end{tcolorbox}

\noindent
\begin{minipage}[t]{0.49\textwidth}
    \begin{tcolorbox}[
        colback=blue!5,
        colframe=blue!70!black,
        title=\textbf{GRASP Rankings},
        fonttitle=\bfseries,
        bottom=1mm
    ]
    \footnotesize
    \texttt{anthropic/claude-haiku-4.5}: 1 \\
    \texttt{deepseek/deepseek-v3.2}: 1 \\
    \texttt{google/gemini-3-flash-preview}: 1 \\
    \texttt{meta-llama/llama-4-scout}: 1 \\
    \texttt{openai/gpt-5.2-chat}: 1 \\
    \texttt{xiaomi/mimo-v2-flash}: 2
    \end{tcolorbox}
\end{minipage}
\hfill
\begin{minipage}[t]{0.49\textwidth}
    \begin{tcolorbox}[
        colback=red!5,
        colframe=red!70!black,
        title=\textbf{RAW Rankings},
        fonttitle=\bfseries,
        bottom=1mm
    ]
    \footnotesize
    \texttt{anthropic/claude-haiku-4.5}: 7 \\
    \texttt{deepseek/deepseek-v3.2}: 17 \\
    \texttt{meta-llama/llama-4-scout}: 5 \\
    \texttt{openai/gpt-5.2-chat}: 20 \\
    \texttt{xiaomi/mimo-v2-flash}: 19 \\
    \textit{(Values indicate wide disagreement)}
    \end{tcolorbox}
\end{minipage}

\paragraph{Analysis.}
The same argument ranges from rank \(5\) to rank \(20\) under \textsc{RAW}, making its priority highly dependent on the judge model. Under \textsc{GRASP}, independently induced interaction graphs yield nearly identical ranks \((1,1,1,1,1,2)\)\footnote{\# Attk and Mean represent the number of attackers of the target argument and their corresponding average strength as obtained from the $W$ constructed by the \texttt{openai/gpt-5.2-chat} judge.}. This matters sociotechnically because model-dependent ranking variation can shape which arguments are surfaced, ignored, or acted upon in downstream deliberative settings. Structural aggregation reduces this arbitrariness by making the basis for prioritization explicit, inspectable, and  tied to the induced attack--defense graph rather than to an opaque holistic verdict.

Additional ablations, prompt specifications, \textsc{GRASP} pseudocode, and analyses of the
relationship between \textsc{GRASP} scores and human convincingness are provided in the Appendix.

\section{Conclusion and Discussion}

We introduced \textbf{\textsc{GRASP}}, a structural aggregation framework that ranks
arguments from pairwise interaction patterns. When LLM judges prioritize arguments,
inter-model variation determines which arguments are surfaced or ignored. \textsc{GRASP}
reduces this arbitrariness by replacing opaque holistic ranking with an explicit pathway
from local attack--defense judgments to global structural strength.

Unlike quality-aware or persuasion-oriented models, \textsc{GRASP} scores arguments by how
they are attacked and defended within the explicit structure, not by truth, rhetorical
appeal, or human convincingness. This positions \textsc{GRASP} as a complementary audit
layer: deliberation systems can use convincingness models to surface persuasive arguments
while flagging cases where structural rankings diverge sharply, marking arguments whose
appeal outpaces their dialectical support.

Empirically, \textsc{GRASP} improves inter-model agreement, reduces worst-case divergence
from consensus, and yields consistent rankings across independently constructed attack
graphs. Its weak association with human convincingness labels reinforces the central
distinction: structural robustness and persuasive success are different evaluation targets.
This matters for auditable LLM-assisted deliberation, where prioritization should be
inspectable rather than hidden inside a model-specific holistic verdict.

\paragraph{Limitations and Future Work.}
\textsc{GRASP} currently requires \(O(n^2)\) pairwise interaction scoring, uses a
graph-derived defense construction whose suitability may vary by domain, and has limited
human validation of structural sufficiency as a distinct annotation target. Future work
includes scalable edge filtering, broader defense formulations, human studies separating
structural robustness from persuasiveness, reward-model applications, and streaming
interaction graphs.

\section*{Acknowledgements}
This research was funded by the Max Planck \& Amazon Science Hub. The authors would further like to thank Hilde Kuehne, Naomi Saphra, Bernhard Schölkopf, Moritz Hardt and Victor May for their invaluable feedback in helping shape this idea into a paper. Antonio Orvieto acknowledges the financial support of the Hector Foundation.

\clearpage
\bibliographystyle{plainnat}
\bibliography{references}

\clearpage
\appendix

\section*{\huge\makebox[\linewidth][c]{Appendix}}
\phantomsection
\addcontentsline{toc}{section}{Appendix}

\vspace{1em}
\section*{Table of Contents}
\phantomsection

\begingroup
\small
\setlength{\parskip}{2pt}
\setlength{\parindent}{0pt}

\hyperref[app:related_work]{A \quad Related Work}\dotfill\pageref{app:related_work}

\hspace{1.5em}\hyperref[app:rw_gradual]{A.1 \quad Gradual semantics and argument ranking}\dotfill\pageref{app:rw_gradual}

\hspace{1.5em}\hyperref[app:rw_mining]{A.2 \quad Argument ranking semantics versus argument mining}\dotfill\pageref{app:rw_mining}

\hspace{1.5em}\hyperref[app:rw_quality]{A.3 \quad Argument quality, persuasion, and sufficiency}\dotfill\pageref{app:rw_quality}

\hspace{1.5em}\hyperref[app:rw_llm]{A.4 \quad LLMs for argumentation and judging}\dotfill\pageref{app:rw_llm}

\hspace{1.5em}\hyperref[app:rw_positioning]{A.5 \quad Positioning}\dotfill\pageref{app:rw_positioning}

\vspace{0.35em}
\hyperref[app:impact]{B \quad Impact Statement}\dotfill\pageref{app:impact}

\hspace{1.5em}\hyperref[app:impact_positive]{B.1 \quad Positive impacts}\dotfill\pageref{app:impact_positive}

\hspace{1.5em}\hyperref[app:impact_risks]{B.2 \quad Risks and limitations}\dotfill\pageref{app:impact_risks}

\hspace{1.5em}\hyperref[app:impact_mitigations]{B.3 \quad Mitigations}\dotfill\pageref{app:impact_mitigations}

\vspace{0.35em}
\hyperref[sec:proof]{C \quad Proofs}\dotfill\pageref{sec:proof}

\hspace{1.5em}\hyperref[app:proofs_prelim]{C.1 \quad Preliminaries}\dotfill\pageref{app:proofs_prelim}

\hspace{1.5em}\hyperref[app:proofs_convergence]{C.2 \quad GRASP convergence}\dotfill\pageref{app:proofs_convergence}

\vspace{0.35em}
\hyperref[sec:app_example]{D \quad Illustrative Example: Dynamic Ranking Shift and Convergence}\dotfill\pageref{sec:app_example}

\vspace{0.35em}
\hyperref[sec:synthetic_structural]{E \quad Structural Evaluation on Synthetic Graphs}\dotfill\pageref{sec:synthetic_structural}

\hspace{1.5em}\hyperref[app:struct_archetypes]{E.1 \quad Structural archetypes}\dotfill\pageref{app:struct_archetypes}

\hspace{1.5em}\hyperref[app:struct_crc]{E.2 \quad Critical ranking conditions}\dotfill\pageref{app:struct_crc}

\hspace{1.5em}\hyperref[subsec:struct_methods]{E.3 \quad Methods and metrics}\dotfill\pageref{subsec:struct_methods}

\hspace{1.5em}\hyperref[app:struct_results]{E.4 \quad Results}\dotfill\pageref{app:struct_results}

\vspace{0.35em}
\hyperref[app:ablations]{F \quad Additional Diagnostic Analyses}\dotfill\pageref{app:ablations}

\hspace{1.5em}\hyperref[sec:centrality]{F.1 \quad Centrality alignment and structural centers}\dotfill\pageref{sec:centrality}

\hspace{1.5em}\hyperref[sec:w_vs_rank_similarity]{F.2 \quad Structural similarity vs. ranking similarity}\dotfill\pageref{sec:w_vs_rank_similarity}

\hspace{1.5em}\hyperref[app:graph_geometry]{F.3 \quad Attack-graph geometry across debates}\dotfill\pageref{app:graph_geometry}

\hspace{1.5em}\hyperref[app:angle_agreement]{F.4 \quad Angle-level agreement analysis}\dotfill\pageref{app:angle_agreement}

\hspace{1.5em}\hyperref[app:idebate_sanity]{F.5 \quad External Sanity Check: Recovering Human-Written Point--Counterpoint Relations}\dotfill\pageref{app:idebate_sanity}

\vspace{0.35em}
\hyperref[app:dataset-details]{G \quad \textsc{StructDebate} Construction Details}\dotfill\pageref{app:dataset-details}

\hspace{1.5em}\hyperref[app:generation_prompts]{G.1 \quad Argument generation prompt schemas}\dotfill\pageref{app:generation_prompts}

\hspace{1.5em}\hyperref[app:motions]{G.2 \quad Debate motions}\dotfill\pageref{app:motions}

\vspace{0.35em}
\hyperref[app:gridsearch]{H \quad Hyperparameter Sensitivity via Cross-Model Agreement}\dotfill\pageref{app:gridsearch}

\vspace{0.35em}
\hyperref[app:prompts]{I \quad Judging Prompts and Prompt Optimization}\dotfill\pageref{app:prompts}

\hspace{1.5em}\hyperref[app:raw_prompt_opt]{I.1 \quad Prompt optimization for RAW rankings}\dotfill\pageref{app:raw_prompt_opt}

\hspace{1.5em}\hyperref[app:prompt_schemas]{I.2 \quad Judging prompt schemas}\dotfill\pageref{app:prompt_schemas}

\vspace{0.35em}
\hyperref[app:case_texts]{J \quad Case Study Details}\dotfill\pageref{app:case_texts}

\vspace{0.35em}
\hyperref[app:contrast_cases]{K \quad Additional Qualitative Case Studies}\dotfill\pageref{app:contrast_cases}

\vspace{0.35em}
\hyperref[sec:ddo]{L \quad Debate Decision Outcomes: Structural Strength vs. Convincingness}\dotfill\pageref{sec:ddo}

\hspace{1.5em}\hyperref[app:ddo_dataset]{L.1 \quad Dataset and filtering}\dotfill\pageref{app:ddo_dataset}

\hspace{1.5em}\hyperref[app:ddo_setup]{L.2 \quad Experimental setup}\dotfill\pageref{app:ddo_setup}

\hspace{1.5em}\hyperref[app:ddo_results]{L.3 \quad Results}\dotfill\pageref{app:ddo_results}

\vspace{0.35em}
\hyperref[app:grasp_pseudocode]{M \quad GRASP Pseudocode}\dotfill\pageref{app:grasp_pseudocode}


\endgroup

\clearpage

\section{Related Work}
\label{app:related_work}

\textsc{GRASP} is primarily situated in gradual and ranking-based semantics for abstract
argumentation. We also discuss related NLP work on argument mining, argument quality, and
LLM-based judging to clarify the boundary of our contribution. Rather than extracting
argumentative structure from text or predicting persuasive quality, \textsc{GRASP} assumes an
explicit interaction graph and studies how its local relations should be aggregated into a
global ranking.

\paragraph{Gradual semantics and argument ranking.}
\label{app:rw_gradual}
Abstract argumentation frameworks (AFs) model arguments as nodes and attacks as directed
edges, with classical semantics defined through accepted \emph{extensions}
\citep{dung1995acceptability,baroni2011introduction}. These semantics are primarily
set-valued: they determine which arguments are jointly acceptable, rather than assigning a
fine-grained strength to each argument. Gradual and ranking-based semantics address this
limitation by assigning arguments numerical or ordinal scores
\citep{cayrol2005graduality,amgoud2013ranking,bonzon2016comparative,baroni2018many}.
Prominent families include categorizer-style propagation methods
\citep{besnard2001logic}, ranking semantics based on attack structure
\citep{amgoud2013ranking,baroni2014input}, and numerical or equilibrium-style approaches
\citep{gabbay2015equilibrium}. Subsequent work extends these ideas to weighted and
quantitative settings~\citep{libman2024abstract}, assumption-based argumentation
\citep{rapberger2025gradual}, extension-derived rankings~\citep{bengel2025extension}, and
power-index-based formulations~\citep{bistarelli2020power}.

\textsc{GRASP} belongs to this ranking-semantics tradition: given an explicit interaction
graph, it maps local attacks and defenses to continuous argument strengths. Its contribution
is a damped attack--defense propagation operator with a contraction-based convergence
guarantee, paired with an empirical study of how such semantics behave when interaction
graphs are induced by LLM judges.

\paragraph{Argument ranking semantics versus argument mining.}
\label{app:rw_mining}
\textsc{GRASP} is an argument ranking semantics, not an argument-mining system. Argument
mining typically addresses the upstream problem of identifying argumentative components and
relations from raw text, including claims, premises, supports, attacks, stance, and discourse
structure~\citep{stab2014identifying,chakrabarty2019ampersand,potash2017towards,
li2025large}. By contrast, ranking semantics assume that an argument graph or weighted
interaction structure is already available, and define how arguments should be scored once
that structure is fixed~\citep{amgoud2013ranking,bonzon2016comparative,baroni2018many}.
This distinction separates two sources of error: graph construction concerns whether the
edges in \(W\) and \(D\) are correct, while ranking semantics concerns how a fixed graph is
aggregated into global strengths.

This distinction also makes \textsc{GRASP} corpus-agnostic. The arguments may be
human-written or machine-generated, and the interaction graph may be produced by humans,
models, or external tools. Once \(W\) and \(D\) are specified, \textsc{GRASP} provides a
deterministic aggregation rule.

\paragraph{Argument quality, persuasion, and sufficiency.}
\label{app:rw_quality}
A large body of NLP work evaluates arguments by convincingness, persuasiveness, or intrinsic
quality~\citep{toledo2019automatic,wachsmuth2020intrinsic,wachsmuth2017computational,
potash2019ranking,li2020exploring,bozdag2025readsystematicsurveycomputational}. Related
work studies how discourse structure and argument arrangement affect persuasion
\citep{stab2014identifying,mirzakhmedova2023unveiling}. These tasks are complementary to
ours: \textsc{GRASP} does not attempt to predict whether a human will find an argument
convincing, but instead measures structural strength within an explicit interaction graph.

The closest conceptual link is to sufficiency-oriented views of argument quality
\citep{cohen2001evaluating,gurcke-etal-2021-assessing}. Global sufficiency asks whether an
argument adequately addresses counterarguments that could reasonably be anticipated.
Structural sufficiency is deliberately narrower: it evaluates robustness only with respect to
attacks instantiated in the graph. This distinction is central to our framing because
\textsc{GRASP} measures graph-relative structural robustness, not persuasive success.

\paragraph{LLMs for argumentation and judging.}
\label{app:rw_llm}
Recent work explores the use of LLMs for argument-related tasks, including argument mining,
argument quality annotation, debate evaluation, and computational argumentation
\citep{chen2024exploring,li2025large,mirzakhmedova2024large,rescala2024can,
li2024argumentation,sanayei2025can}. However, LLM-based evaluators can be inconsistent or
biased, and neural argument models may exploit shallow cues rather than genuine reasoning
\citep{niven2019probing,taubenfeld2024systematic,stureborg2024large,hu2024explaining}.
\textsc{GRASP} follows a hybrid strategy: LLMs estimate local pairwise interactions, while a
formally defined operator performs global strength propagation. This aligns with work on
inferring attack relations for gradual semantics~\citep{oren2023inferring} and broader
neural--symbolic approaches~\citep{yu2025explain}, but shifts the focus from holistic
judging to auditable structural aggregation.

\paragraph{Positioning.}
\label{app:rw_positioning}
Overall, \textsc{GRASP} should be understood as a structural ranking semantics for explicit
or LLM-induced interaction graphs. It is not a model of truth, persuasion, or human
preference; rather, it provides an explicit pathway from local attack--defense structure to
global argument rankings. Our empirical results support the view that structural robustness
and persuasive effectiveness are distinct, making \textsc{GRASP} complementary to argument
quality and convincingness models.

\section{Impact Statement}
\label{app:impact}

This work studies how argument rankings can be derived from explicit interaction structure
rather than from opaque holistic judgments. By introducing \textsc{GRASP}, we provide an
auditable operator for aggregating local attack--defense relations into global argument
strengths. The intended contribution is methodological: \textsc{GRASP} helps distinguish
structural robustness from persuasion, rhetoric, factuality, or human preference.

\paragraph{Positive impacts.}
\label{app:impact_positive}
\textsc{GRASP} may support tools for analyzing debates, policy discussions, legal reasoning,
scientific claims, and multi-agent deliberation. Because the ranking is computed from an
explicit interaction graph, users can inspect which attacks and defenses contribute to an
argument's score. This may improve transparency in argument evaluation and provide a
complementary signal to persuasion-oriented or quality-aware models.

\paragraph{Risks and limitations.}
\label{app:impact_risks}
\textsc{GRASP} does not measure truth, ethical correctness, factual validity, societal
desirability, or persuasive effectiveness. Misinterpreting structural strength as any of
these properties could lead to inappropriate reliance, especially in high-stakes domains.
Moreover, if the interaction graph is induced by biased, noisy, or low-quality scorers, the
resulting rankings may inherit those errors.

\paragraph{Mitigations.}
\label{app:impact_mitigations}
\textsc{GRASP} should be used as a structural analysis tool, not as a standalone decision
oracle. In practical settings, it should be paired with factual verification, domain expertise,
human review, and explicit auditing of the underlying interaction graph. Its modular design
makes such auditing possible: users can inspect, revise, or replace the local interaction
scores before aggregation.

Overall, the main societal value of \textsc{GRASP} is to clarify what can and cannot be
inferred from argument structure. It offers a transparent way to analyze dialectical
robustness while making clear that structural strength is distinct from truth, persuasion,
and normative correctness.

\section{Proofs}
\label{sec:proof}

We begin by recalling that convergence of iterative schemes is most naturally studied through contraction properties of the underlying operator. In particular, if the GRASP operator is a contraction on a suitable domain, then classical fixed-point results guarantee both existence and uniqueness of an equilibrium, as well as convergence from arbitrary initialization.

\subsection{Preliminaries}
\label{app:proofs_prelim}
We provide here basic results, for a modern reference plase refer to~\cite{bullo2022contraction}.
\begin{definition}[Metric Space] 
    Let $\set$ be a non-empty set, a map $d:\set\times \set\to\R$ is a metric if (1) $d(x,y)=0\iff x=y$, (2) $\forall x,y\in \set, \ d(x,y)=d(y,x)$, (3)  $\forall x,y,z\in \set, \ d(x,y)\le d(x,z)+d(y,z)$.
\end{definition}

\begin{definition}[Cauchy/Convergent Sequences] Let $\{x_k\}_{k\in\N}$ be a sequence in $(\set,d)$. We call $\{x_k\}_{k\in\N}$ Cauchy if for any $\epsilon>0$ there exist $k$ such that, for all $h\in\N$ and $i\ge k$, $d(x_i,x_{i+h})\le\epsilon$. We call $\{x_k\}_{k\in\N}$ convergent to $x^*\in \set$ if for any $\epsilon>0$ there exist $k$ such that $d(x_i,x^*)\le\epsilon$ for all $i\ge k$. 
\end{definition}

\begin{definition}[Complete Metric Space] $(\set,d)$ is complete if every Cauchy sequence in $\set$ converges to a point in $\set$. $(\R^d, \|\cdot\|)$ is complete.
\end{definition}

\begin{definition}[Lipschitz Maps, Contraction]
    Let $(\set,d)$ be a metric space, $G:\set\to \set$ is Lipschitz if there exists a $\ell\ge0$~(Lipschitz constant for $G$) such that for all $x,y\in X$, $d(G(x),G(y))\le\ell d(x,y)$. If $\ell<1$ is possible, $G$ is called a contraction.
\end{definition}

This well-known result is the workhorse of most convergence analyses.
\begin{theorem}[Banach Contraction Theorem] Let $(\set,d)$ be a complete metric space, and $G$ a contraction with factor $\ell$. Then, $G$ has a unique fixed point $s^*\in \set$, and the sequence generated by $s_{k+1}=G(s_k)$ converges to $s^*$.
\label{thm:banach}
\end{theorem}

\subsection{GRASP Convergence}
\label{app:proofs_convergence}

Recall that GRASP operator $G : \mathbb{R}^n \to \mathbb{R}^n$ is defined coordinatewise as
\begin{equation}
G(s)_j
\;=\;
\frac{1 + \beta \sum_k D_{kj} s_k}
     {1 + \alpha \sum_i W_{ij} s_i},
\end{equation}
where $W_{ij}, D_{ij} \ge 0$ for all $i,j \in [n]$.

\paragraph{Contraction on compact subspaces.} While global contraction on the entire space is often too strong to expect, many nonlinear operators exhibit contraction behavior when restricted to an invariant subset. This motivates analyzing GRASP on a bounded region of interest, corresponding to a meaningful range of argument strengths.

The following lemma formalizes a standard extension of Banach’s theorem to such invariant subsets.
\begin{lemma}[Banach on subsets]
    Let $(\set,d)$ be complete, and assume there is a point $s_0$ and a radius $r$ such that $G:\set\to\set$ is a contraction with Lipschitz constant $\ell$ on $B:=\{x\in\set, d(x,s_0)\le r\}$. Assume $d(s_0,G(s_0))\le r(1-\ell)$, then $B$ is invariant under $G$ and the contraction theorem applies to $G$ restricted to $B$.
\end{lemma}
\begin{proof} It is sufficient to show that, for $x\in B$, $G(x)\in B$. Note that
\begin{align*}
    d(G(x),s_0)\le d(G(x),G(s_0))+d(G(s_0),s_0)\\ \le \ell d(x,s_0)+ (1-\ell)r\le r
\end{align*}
and hence this concludes the proof.
\end{proof}

\paragraph{On damping in GRASP.} 

Since GRASP incorporates a damping step in practice, it is important to understand whether damping alone can induce contraction. The following remark clarifies that damping cannot compensate for a lack of contraction in the base operator itself.

\begin{remark}[Damping cannot turn a non-contraction into a contraction]
In GRASP, we use damping:
\[
s^+ = \alpha s + (1-\alpha)G(s) =:\hat G(s),
\]
Doing cannot make $\ell\le 1$ if it was not previously so. Indeed, on $(\R^d,\|\cdot\|)$, for an arbitrary norm:
\begin{align*}
    \|\hat G(x)-\hat G(y)\|=  \|\alpha (x-y) + (1-\alpha)(G(x)-G(y))\|\\
    \le (\alpha + (1-\alpha)\ell) \|x-y\|.
\end{align*}
So if $\ell\le 1$, the new factor is also $\le 1$. If $\ell\ge 1$, the new factor is also $\ge 1$.
\label{rmk:damping}
\end{remark}
The remark above shows that to successfully characterize the convergence properties of (undamped) GRASP, it is necessary and sufficient to consider the properties of $ s\mapsto G(s)$.

\paragraph{Approach.}  Our approach proceeds in two steps. First, we identify a natural bounded subset of $\mathbb{R}^d$ that is invariant under $G$. Second, we establish that $G$ is a contraction on this set under explicit conditions on the interaction matrices.

\subsubsection{Invariance}
\label{sec:grasp_convergence}

Throughout this section, we work with the set
\[
\set := \{ s \in \mathbb{R}^d \;:\; \|s - \mathbf{1}\|_\infty \le 1 \},
\]
which corresponds to bounded, nonnegative strength vectors centered around the neutral baseline $\mathbf{1}=(1,\ldots,1)$.

\begin{lemma}
    Let $x,y\in\mathbb{R}^d$. Then
    \[
    \|x \odot y\|_\infty \le \|x\|_\infty \|y\|_\infty.
    \]
    \label{lemma:int_had}
\end{lemma}
\begin{proof}
    Let $x,y\in\mathbb{R}^d$ and define $(x \odot y)_i := x_i y_i$. Then
    \begin{align*}
    \|x \odot y\|_\infty
    &= \max_i |x_i y_i| \\
    &\le \max_i \bigl(|x_i| \, \|y\|_\infty \bigr) \\
    &= \|y\|_\infty \max_i |x_i| \\
    &= \|x\|_\infty \|y\|_\infty,
    \end{align*}
    where we used the bound $|y_i| \le \|y\|_\infty$ for all $i$.
\end{proof}

We now show that $\set$ is invariant under the GRASP operator. This guarantees that once the iteration enters $\set$, it remains there for all subsequent steps.

\begin{lemma}[Invariance]
    Let $\set:=\{s\in\R^d, \|s-1\|_\infty\le 1\}$ and let $G:\set\to\R^d$ be defined elementwise as
    \[
    G(s)_i = \frac{1 + \beta (D^\top s)_i}{1 + \alpha(W^\top s)_i},
    \]
    for matrices $W,D\in\R^{d\times d}$ and scalars $\alpha,\beta\ge 0$. If $W,D$ have non-negative entries and
    \[
    \alpha \le \frac{1}{4\|W\|_1},
    \qquad
    \beta \le \frac{1}{4\|D\|_1},
    \]
    then $G(\set)\subseteq \set$.
    \label{lemma:invariance}
\end{lemma}
\begin{proof}
    Using the triangle inequality and denoting $\mathbf{1}=(1,\ldots,1)\in\set$, we obtain
    \begin{align*}
    \|G(s)-\mathbf{1}\|_\infty
    &= \left\|
        \frac{1 + \beta D^\top s}{1 + \alpha W^\top s}
        - \mathbf{1}
       \right\|_\infty \\
    &= \left\|
        \frac{(\beta D^\top - \alpha W^\top)s}
             {1 + \alpha W^\top s}
       \right\|_\infty.
    \end{align*}
    Since $\alpha>0$ and $W$ has non-negative entries, the denominator satisfies
    $1+\alpha W^\top s > 1$ elementwise and can be dropped for an upper bound. Using norm subadditivity,
    \begin{align*}
    \|G(s)-\mathbf{1}\|_\infty
    &\le \|(\beta D^\top - \alpha W^\top)s\|_\infty \\
    &\le \|\beta D^\top - \alpha W^\top\|_\infty \|s\|_\infty \\
    &\le (\beta\|D^\top\|_\infty + \alpha\|W^\top\|_\infty)\|s\|_\infty.
    \end{align*}
    Using the identity $\|M^\top\|_\infty = \|M\|_1$ and the fact that $s\in\set$ implies $\|s\|_\infty\le 2$, we conclude
    \[
    \|G(s)-\mathbf{1}\|_\infty
    \le (\beta\|D\|_1 + \alpha\|W\|_1)\cdot 2
    \le \left(\tfrac14 + \tfrac14\right)\cdot 2
    = 1.
    \]
\end{proof}

\subsubsection{Lipschitz Constant}

We next establish a Lipschitz bound for $G$ on $\set$, which will allow us to invoke Banach’s fixed-point theorem.

\begin{lemma}[Lipschitz Constant]
    Consider the GRASP operator defined coordinatewise as
    \[
    G(s)_i = \frac{1 + \beta (D^\top s)_i}{1 + \alpha(W^\top s)_i}.
    \]
    Let $x,y\in\set$. If $W,D$ have non-negative entries, then
    \begin{align*}
    \|G(x)-G(y)\|_\infty
    \le \ell \|x-y\|_\infty, \\
    \ell
    := \beta\|D\|_1
      + \alpha\|W\|_1
        \frac{\|G(x)\|_\infty + \|G(y)\|_\infty}{2}.
    \end{align*}
    \label{lemma:lip}
\end{lemma}
\begin{proof}
    Define elementwise
    \[
    (\NN_x)_i := 1 + \beta(D^\top x)_i,
    \qquad
    (\DD_x)_i := 1 + \alpha(W^\top x)_i.
    \]
    Then
    \begin{align*}
    \|G(x)-G(y)\|_\infty
    &= \left\|
        \frac{\NN_x}{\DD_x}
        - \frac{\NN_y}{\DD_y}
       \right\|_\infty \\
    &= \left\|
        \frac{\NN_x \DD_y - \NN_y \DD_x}{\DD_x \DD_y}
       \right\|_\infty.
    \end{align*}
    Expanding the numerator in two symmetric ways,
    \begin{align*}
    \NN_x \DD_y - \NN_y \DD_x
    &= (\NN_x-\NN_y)\DD_y + (\DD_y-\DD_x)\NN_y, \\
    \NN_x \DD_y - \NN_y \DD_x
    &= (\NN_x-\NN_y)\DD_x + (\DD_y-\DD_x)\NN_x.
    \end{align*}
    Averaging these expressions yields
    \begin{align*}
        |G(x)-G(y)\|_\infty \\ = \left\|
        (\NN_x-\NN_y)\frac{\DD_x+\DD_y}{2\DD_x\DD_y}
        + (\DD_y-\DD_x)\frac{\NN_x+\NN_y}{2\DD_x\DD_y}
      \right\|_\infty
    \end{align*}

    Applying Lemma~\ref{lemma:int_had} and the triangle inequality,
    \begin{align*}
    \|G(x)-G(y)\|_\infty
    &\le \|\NN_x-\NN_y\|_\infty
         \left\|\frac{\DD_x+\DD_y}{2\DD_x\DD_y}\right\|_\infty \\
    &\quad + \|\DD_y-\DD_x\|_\infty
         \left\|\frac{\NN_x+\NN_y}{2\DD_x\DD_y}\right\|_\infty.
    \end{align*}
    Since $\DD_x,\DD_y\ge 1$ elementwise,
    \[
    \left\|\frac{\DD_x+\DD_y}{2\DD_x\DD_y}\right\|_\infty \le 1.
    \]
    Moreover,
    \begin{align*}
    \left\|\frac{\NN_x+\NN_y}{2\DD_x\DD_y}\right\|_\infty
    &\le \frac12
      \left(
        \left\|\frac{\NN_x}{\DD_x}\right\|_\infty
        +
        \left\|\frac{\NN_y}{\DD_y}\right\|_\infty
      \right) \\
    &= \frac{\|G(x)\|_\infty + \|G(y)\|_\infty}{2}.
    \end{align*}
    Finally,
    \[
    \|\NN_x-\NN_y\|_\infty
    \le \beta\|D^\top\|_\infty \|x-y\|_\infty
    = \beta\|D\|_1 \|x-y\|_\infty,
    \]
    and similarly
    \[
    \|\DD_x-\DD_y\|_\infty
    \le \alpha\|W\|_1 \|x-y\|_\infty.
    \]
\end{proof}

The convergence result in the main paper naturally follows.


\subsubsection{Proof of Theorem~\ref{thm:grasp}}

\paragraph{Proof sketch.}
The proof identifies a bounded invariant set
\(\set=\{s\in\mathbb{R}^d:\|s-\mathbf{1}\|_\infty\le 1\}\) and then shows that \(G\) is a
contraction on \(\set\) under the infinity norm. Under the nonnegativity assumptions on
\(W,D\) and the bounds
\(\alpha \le 1/(4\|W\|_1)\) and \(\beta \le 1/(4\|D\|_1)\), Lemma~\ref{lemma:invariance}
implies \(\|G(s)-\mathbf{1}\|_\infty \le 1\) for all \(s\in\set\), hence
\(G(\set)\subseteq \set\). Next, Lemma~\ref{lemma:lip} bounds the change in the ratio
defining \(G\): for any \(x,y\in\set\), the perturbations induced by the attack and defense
terms yield
\[
\|G(x)-G(y)\|_\infty \le \ell \|x-y\|_\infty,
\qquad
\ell \le \beta\|D\|_1 + 2\alpha\|W\|_1 \le \frac{3}{4}<1.
\]
Thus \(G\) is a contraction on the invariant set \(\set\). Banach's fixed-point theorem then
gives existence, uniqueness, and convergence to the fixed point.

\begin{proof}
    The invariance claim follows directly from Lemma~\ref{lemma:invariance}. The contraction
    claim follows from Lemma~\ref{lemma:lip}: under the assumptions on $\alpha,\beta$,
    \begin{align*}
        \ell
        &= \max_{x,y\in\set}\left[
        \beta\|D\|_1
        + \alpha\|W\|_1
        \frac{\|G(x)\|_\infty+\|G(y)\|_\infty}{2}
        \right] \\
        &\le \beta\|D\|_1 + 2\alpha\|W\|_1
        \le \frac{1}{4}+\frac{1}{2}<1,
    \end{align*}
    where we used that $\|G(s)\|_\infty\le 2$ for all $s\in\set$. Hence \(G\) is a
    contraction on the invariant set \(\set\), and the result follows from Banach's
    fixed-point theorem.
\end{proof}

\section{Illustrative Example: Dynamic Ranking Shift and Convergence}
\label{sec:app_example}

\begin{wrapfigure}{r}{0.52\textwidth}
\centering
\vspace{-4mm}
\includegraphics[width=0.50\textwidth]{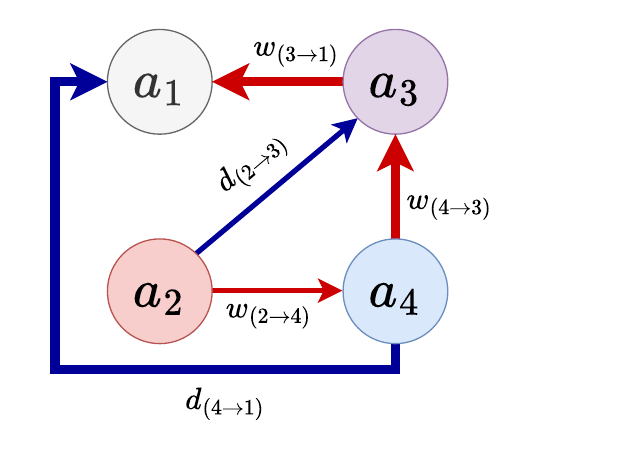}
\caption{Toy argumentation graph illustrating non-monotonic score evolution and a rank
reversal during \textsc{GRASP} iteration. Argument \(a_4\) attacks \(a_3\), thereby defending
\(a_1\); \(a_2\) attacks \(a_4\), indirectly restoring strength to \(a_3\).}
\label{fig:app_example2}
\vspace{-8mm}
\end{wrapfigure}

This example illustrates how \textsc{GRASP} can exhibit non-monotonic score evolution before
converging to a stable ordering, even in a small graph. We consider four arguments:
\(a_1\) is the central claim, \(a_3\) attacks \(a_1\), \(a_4\) attacks \(a_3\) and thereby
defends \(a_1\), and \(a_2\) weakly attacks \(a_4\).

The weighted attack matrix \(W\) uses rows as attackers and columns as targets:
\[
W =
\begin{bmatrix}
0 & 0 & 0 & 0 \\
0 & 0 & 0 & 0.3 \\
1 & 0 & 0 & 0 \\
0 & 0 & 1 & 0
\end{bmatrix},
\]
\[
w_{31}=1.0,\quad w_{43}=1.0,\quad w_{24}=0.3 .
\]
The induced two-hop defense matrix is
\[
D = W W =
\begin{bmatrix}
0 & 0 & 0 & 0 \\
0 & 0 & 0.3 & 0 \\
0 & 0 & 0 & 0 \\
1 & 0 & 0 & 0
\end{bmatrix},
\]
where \(d_{41}=1.0\) captures that \(a_4\) defends \(a_1\), and \(d_{23}=0.3\) captures
that \(a_2\) weakly defends \(a_3\) by attacking \(a_4\).

We use the undamped update
\[
G(s)_j =
\frac{1 + \beta (D^\top s)_j}
     {1 + \alpha (W^\top s)_j},
\qquad
s^{(t+1)} = G(s^{(t)}),
\]
with \(\alpha=1.0\), \(\beta=0.5\), and \(s^{(0)}=\mathbf{1}\). This value of
\(\alpha\) lies outside the sufficient contraction bound in Theorem~\ref{thm:grasp}; the
theorem gives a worst-case sufficient condition, not a necessary condition for convergence.
The iteration nevertheless converges in this graph.

\begin{center}
\small
\setlength{\tabcolsep}{6pt}
\renewcommand{\arraystretch}{1.08}
\begin{tabular}{ccl}
\toprule
\textbf{Iteration} & \textbf{Scores} & \textbf{Ranking} \\
\midrule
1 & \([0.750,\;1.000,\;0.575,\;0.769]\) & \(a_2 > a_4 > a_1 > a_3\) \\
2 & \([0.879,\;1.000,\;0.650,\;0.769]\) & \(a_2 > a_1 > a_4 > a_3\) \\
3 & \([0.839,\;1.000,\;0.650,\;0.769]\) & \(a_2 > a_1 > a_4 > a_3\) \\
\bottomrule
\end{tabular}
\end{center}

The rank reversal occurs between iterations 1 and 2: \(a_1\) overtakes \(a_4\).
Initially, \(a_1\) is suppressed by the strong attack from \(a_3\). As \(a_4\) weakens
\(a_3\), pressure on \(a_1\) decreases and \(a_1\) recovers. Meanwhile, the weak but
persistent attack from \(a_2\) fixes \(a_4\)'s score below the neutral baseline. The score
of \(a_1\) briefly overshoots at iteration 2 before settling to its fixed-point value,
illustrating that \textsc{GRASP} dynamics are not equivalent to a one-step local sorting
heuristic.

\section{Structural Evaluation on Synthetic Graphs}
\label{sec:synthetic_structural}

To evaluate whether ranking operators respect the principles of
\emph{structural sufficiency} (Section~\ref{sec:struct_suff}), we construct a controlled
synthetic testbed of argumentation graphs with explicit interaction structure. Unlike
natural debates, these graphs admit unambiguous structural constraints, allowing direct
evaluation against \emph{necessary ranking conditions} implied by structural sufficiency
rather than subjective judgment.

Our objective is not to induce a total order, but to test whether a method violates ranking
relations that should hold when one argument is structurally better defended than another.

\subsection{Structural Archetypes}
\label{app:struct_archetypes}

We consider a suite of small, canonical \emph{structural archetypes}
(Figure~\ref{fig:structural_motifs}), each isolating a distinct dialectical phenomenon.
Each instance is a directed argumentation graph \(\mathcal{G}=(A,R^-)\) with
\(|A|\in[4,6]\). Edges represent attacks with positive weight; self-attacks are disallowed.
To avoid degenerate sparsity, we optionally add low-weight noise attacks that preserve the
core structure while increasing heterogeneity.

\begin{wrapfigure}{r}{0.55\textwidth}
\centering
\vspace{-4mm}
\includegraphics[width=0.53\textwidth]{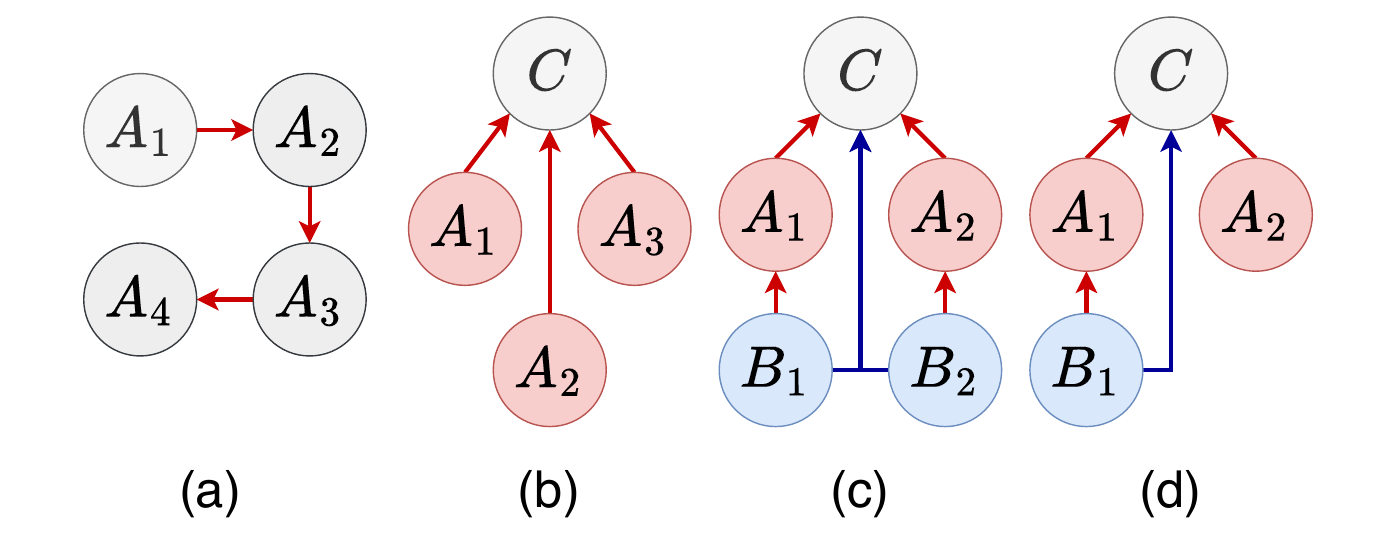}
\caption{Canonical structural archetypes used in synthetic evaluation.
(a) \textbf{Attack Chain}: counter-attacks reinstate upstream arguments.
(b) \textbf{Fork}: multiple arguments attack a single target.
(c) \textbf{Diamond}: parallel attacks followed by convergent counter-attacks.
(d) \textbf{Bipolar Structure}: multiple attackers of a claim with a downstream attack on one attacker.}
\label{fig:structural_motifs}
\vspace{-4mm}
\end{wrapfigure}

The archetypes include: (i) \textbf{Attack Chains}, testing reinstatement through
counter-attacks; (ii) \textbf{Forks}, testing sensitivity to convergent attack;
(iii) \textbf{Diamonds}, testing cascading defense; and (iv) \textbf{Bipolar Structures},
testing locality under mixed attack patterns.

\paragraph{Random DAG stress tests.}
In addition to canonical motifs, we include random directed acyclic graphs (DAGs) with
\(n=20\) nodes and edge probabilities \(p\in\{0.1,0.3\}\). Edge weights are sampled
uniformly from \([0.2,1.0]\). These graphs do not encode specific dialectical motifs and
are used solely to test robustness and convergence under heterogeneous structure.

\subsection{Critical Ranking Conditions}
\label{app:struct_crc}

For each archetype, we derive \emph{critical ranking conditions} (CRCs), expressed as
pairwise constraints \(a \succ b\). CRCs are necessary conditions implied by structural
sufficiency and intentionally do not define a total order. Each CRC is tied to one of the
axioms in Section~\ref{sec:struct_suff}: attack sensitivity (S1), defense reinstatement
(S2), structural locality (S3), or baseline sufficiency (S4).

For example, in attack-chain and diamond motifs, a defended argument should outrank the
attacker that has been counter-attacked; in fork structures, a target with unneutralized
attackers should not outrank its relevant attackers; and in baseline cases, unattacked
arguments should not be ranked below attacked ones. For random DAGs, only baseline
sufficiency (S4) constraints are imposed. Violations on these graphs therefore reflect
robustness limitations under noisy, non-motif structure rather than failure on canonical
axiomatic cases.

CRCs deliberately exclude stronger monotonic or comparative claims, such as ``adding support
must increase strength,'' because these are not implied by structural sufficiency alone.

\subsection{Methods and Metrics}
\label{subsec:struct_methods}

We compare \textsc{GRASP} against standard structural ranking baselines. All methods operate
solely on the weighted attack matrix \(W \in \mathbb{R}_{\ge 0}^{n \times n}\), where
\(W_{ij}\) denotes the strength of the attack from \(a_i\) to \(a_j\). Each method produces a
real-valued strength score \(s_j\) for every argument \(a_j\); rankings are obtained by
sorting arguments in decreasing \(s_j\).

\paragraph{GRASP variants.}
We evaluate several choices of the defense matrix \(D\) to test how different notions of
defense affect structural consistency. The default choice is \(D=W^2\), where a two-hop path
\(a_k \rightarrow a_i \rightarrow a_j\) means that \(a_k\) attacks an attacker of \(a_j\).
This directly matches the structural sufficiency notion of neutralization by counter-attack.
We also evaluate \(D=W^\top\), which treats reciprocal attack structure as defense;
\(D=W^4\), which uses longer even-length reinstatement chains; and
\(D=W^2+\frac{1}{2}W^4\), which combines direct two-hop defense with discounted higher-order
defense. Finally, \(\beta=0\) disables defense propagation, reducing GRASP to an
attack-only propagation baseline.

\paragraph{Baselines.}
The H-categorizer~\citep{besnard2001logic} penalizes arguments by total incoming attack:
\[
s_j = \frac{1}{1+\sum_i W_{ij}}.
\]
KatzAttack adapts Katz centrality~\citep{katz1953new} by accumulating discounted attack
paths,
\[
c = (I-\lambda W^\top)^{-1}\mathbf{1}, \qquad s_j = 1/c_j,
\]
with \(\lambda\) chosen small enough for convergence. The Defense Ratio baseline uses a
closed-form attack--defense balance,
\[
s_j =
\frac{1+\sum_k (W^2)_{kj}}
     {1+\sum_i W_{ij}},
\]
but does not perform iterative strength propagation. We also include Binary Indegree,
\[
s_j = \frac{1}{1+\sum_i \mathbb{I}[W_{ij}>0]},
\]
and Max Incoming Attack,
\[
s_j = \frac{1}{1+\max_i W_{ij}}.
\]

\paragraph{Metrics.}
Given CRCs of the form \(a \succ b\), a violation occurs whenever \(s_a \le s_b\). We
report: (i) \textbf{violation rate}, the fraction of violated CRCs; (ii) \textbf{violation
severity}, the mean normalized margin \(s_b-s_a\) over violations; (iii) \textbf{mean
iterations} to convergence for iterative methods; and (iv) \textbf{convergence fraction},
the fraction of graphs on which the method converges within the iteration budget.

\subsection{Results}
\label{app:struct_results}

\begin{wraptable}{r}{7.8cm}
\centering
\small
\setlength{\tabcolsep}{3pt}
\aboverulesep=0ex
\belowrulesep=0ex
\resizebox{7.6cm}{!}{%
\begin{tabular}{l cccc}
\toprule
\textbf{Method} & \textbf{Viol.} $\downarrow$ & \textbf{Sev.} $\downarrow$ & \textbf{Iter.} & \textbf{Conv.} \\
\midrule
\rowcolor{gray!15}
\textbf{GRASP} ($D=W^2$)  & \textbf{0.003} & \textbf{0.010} & 65.8 & 100\% \\
GRASP ($D=W^T$)           & 0.163          & 0.091          & 57.6 & 100\% \\
GRASP ($D=W^4$)           & 0.220          & 0.117          & 60.6 & 100\% \\
GRASP ($D=W^2 + \frac{1}{2}W^4$) & 0.269   & 0.129          & 61.2 & 100\% \\
GRASP ($\beta=0$)         & 0.071          & 0.019          & 63.7 & 100\% \\
\midrule
Defense Ratio             & 0.042          & 0.042          & --   & -- \\
KatzAttack                & 0.228          & 0.012          & --   & -- \\
H-Categorizer             & 0.228          & 0.042          & --   & -- \\
Binary Indegree           & 0.233          & 0.031          & --   & -- \\
Max Incoming Attack       & 0.290          & 0.013          & --   & -- \\
\bottomrule
\end{tabular}
}
\caption{Structural evaluation. Viol. and Sev. denote violation rate and violation severity
against critical ranking conditions. Bold indicates best performance; gray shading marks the
default GRASP variant with \(D=W^2\).}
\vspace{-4mm}
\label{tab:synthetic_results}
\end{wraptable}

\paragraph{Summary.}
Table~\ref{tab:synthetic_results} shows that \textsc{GRASP} with \(D=W^2\) achieves the
lowest violation rate and severity across the synthetic suite. This supports the choice of
\(D=W^2\) as the default defense construction for debate-style attack graphs: it directly
encodes the minimal reinstatement pattern in structural sufficiency, namely attacking an
attacker. Alternative defense choices perform worse, suggesting that not all higher-order or
reciprocal constructions align with the CRCs induced by structural sufficiency.

Disabling defense propagation (\(\beta=0\)) increases violations, confirming that attack-only
aggregation is insufficient. The Defense Ratio baseline performs competitively but still
violates more CRCs than GRASP, indicating that non-iterative two-hop defense captures part,
but not all, of the relevant structure. Overall, these results show that structural
sufficiency is best captured by propagation over explicit attack--defense structure rather
than by local, linear, or non-propagative aggregation alone.

\section{Additional Diagnostic Analyses}
\label{app:ablations}

\subsection{Centrality Alignment and Structural Centers}
\label{sec:centrality}

\begin{wraptable}{r}{7.5cm}
\centering
\small
\vspace{-5mm}
\setlength{\tabcolsep}{3pt}
\aboverulesep=0ex
\belowrulesep=0ex
\resizebox{7.3cm}{!}{%
\begin{tabular}{l cc cc}
\toprule
& \multicolumn{2}{c}{Mean \(\rho\) \(\downarrow\)} & \multicolumn{2}{c}{Median \(\rho\) \(\downarrow\)} \\
\cmidrule(lr){2-3} \cmidrule(lr){4-5}
Method & P & MT & P & MT \\
\midrule
RAW                         & -0.007          & \phantom{-}0.071 & \phantom{-}0.049 & \phantom{-}0.093 \\
\rowcolor{gray!15}
\textbf{GRASP}              & -0.955          & -0.960            & -0.967            & -0.973 \\
GRASP-\(W_1\)               & \textbf{-1.000} & \textbf{-0.998}   & \textbf{-1.000}   & \textbf{-0.998} \\
GRASP-\(W_\infty{+}\bar{D}\)& -0.996          & -0.994            & -0.997            & -0.997 \\
GRASP-\(W_1{+}\bar{D}\)     & -0.986          & -0.982            & -0.990            & -0.988 \\
GRASP-\(W_\infty\)          & -0.954          & -0.958            & -0.966            & -0.973 \\
\bottomrule
\end{tabular}
}
\caption{Spearman correlation between rankings and in-strength ordering. More negative values
indicate stronger alignment with the principle that heavily attacked arguments should rank lower.
Gray shading marks the default \textsc{GRASP} variant.}
\vspace{-6mm}
\label{tab:centrality_alignment}
\end{wraptable}

We first test whether \textsc{GRASP} rankings track simple graph-theoretic signals derived
from the attack matrix \(W\). This analysis is intended as a diagnostic sanity check: if
\textsc{GRASP} is structurally grounded, its rankings should reflect the attack geometry of
the graph. It should not, however, be interpreted as evidence that \textsc{GRASP} is
equivalent to a centrality measure; the structural testbed in
Appendix~\ref{sec:synthetic_structural} shows that local and non-propagative baselines
violate substantially more structural constraints.

\begin{wraptable}{r}{6.5cm}
\centering
\small
\vspace{-4mm}
\setlength{\tabcolsep}{6pt}
\aboverulesep=0ex
\belowrulesep=0ex
\begin{tabular}{lccc}
\toprule
Setting & Mean & Med. & 90\% \\
\midrule
Multi-turn & 0.526 & 0.500 & 0.833 \\
Pool       & 0.463 & 0.417 & 0.667 \\
\bottomrule
\end{tabular}
\caption{Consensus over the central argument. Values show the fraction of judge models selecting
the same top-1 argument by incoming attack mass.}
\vspace{-6mm}
\label{tab:central_node_consensus}
\end{wraptable}

For each debate and judge model, we construct a weighted directed graph with edge weights
given by pairwise attack scores. We compute several graph summaries, including in-strength,
out-strength, net-strength, and PageRank. Among these, in-strength---the total incoming attack
mass---is the most stable diagnostic, and we therefore report alignment with in-strength.

\paragraph{Ranking--centrality alignment.}
Table~\ref{tab:centrality_alignment} reports Spearman correlation between each ranking and
the in-strength ordering. All \textsc{GRASP} variants exhibit strong negative correlation
with in-strength, indicating that arguments receiving more incoming attack mass are generally
ranked lower. In contrast, \textsc{RAW} LLM rankings show near-zero correlation, suggesting
that direct holistic judgments do not consistently track this basic structural signal.

\paragraph{Existence of a structural center.}
We also ask whether independently constructed attack graphs agree on which argument is most
structurally central. For each debate and judge model, we identify the top-1 argument by
incoming attack mass and measure the fraction of models selecting the same argument.
Table~\ref{tab:central_node_consensus} reports the resulting consensus rates.

\paragraph{Takeaway.}
Attack graphs induced by different judge models exhibit a moderately stable structural
center. Combined with the strong negative alignment between \textsc{GRASP} rankings and
incoming attack mass, this indicates that \textsc{GRASP} tracks explicit graph structure
whereas \textsc{RAW} rankings do not. The structural testbed further shows that this
structural sensitivity is not sufficient on its own: propagation and defense are needed to
satisfy the critical ranking conditions induced by structural sufficiency.

\subsection{Attack-Graph Geometry Across Debates}
\label{app:graph_geometry}

\begin{figure*}[t]
\centering
\includegraphics[width=\linewidth]{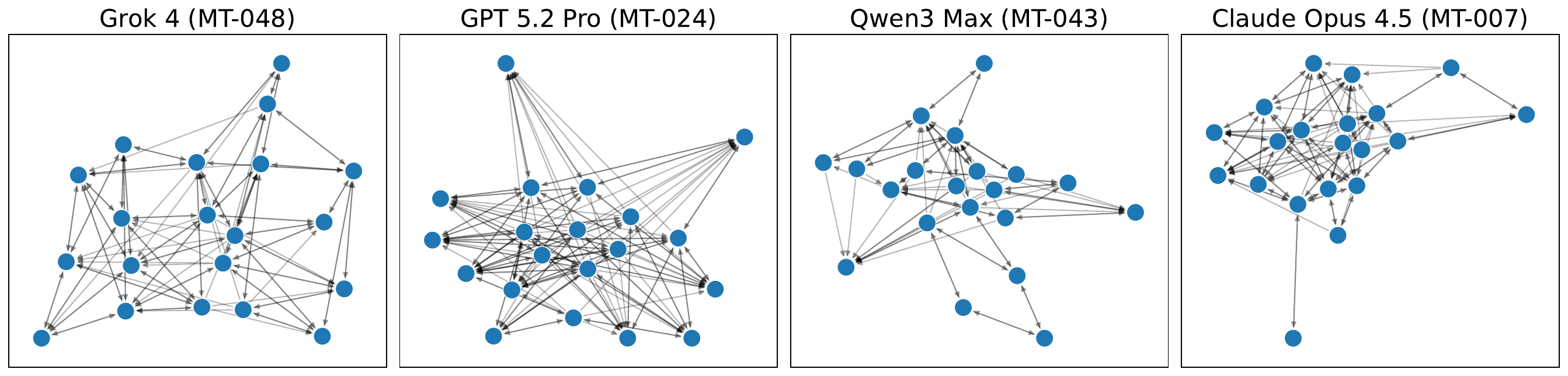}
\caption{
Attack graphs induced by the same judge model, \texttt{openai/gpt-5.2-chat}, for four
multi-turn debates. Nodes correspond to arguments and directed edges indicate attacks with
\(W_{ij}>\tau\), using visualization threshold \(\tau=0.6\).
}
\label{fig:graphs_appendix}
\end{figure*}

We provide qualitative visualizations of induced attack graphs for four multi-turn debates,
all constructed using the same judge model, \texttt{openai/gpt-5.2-chat}. Holding the judge
fixed isolates variation arising from debate content rather than from model-specific scoring
differences.

\paragraph{Debates and motions.}
The four debates are: \textbf{MT-048} (\texttt{x-ai/grok-4}), \emph{This House would break
up dominant technology monopolies}; \textbf{MT-024} (\texttt{openai/gpt-5.2-pro}),
\emph{This House would ban the private ownership of historical artifacts}; \textbf{MT-043}
(\texttt{qwen/qwen3-max}), \emph{This House would allow preventive detention for credible
terrorism threats}; and \textbf{MT-007} (\texttt{anthropic/claude-opus-4.5}), \emph{This
House would abolish the minimum wage law}.

\paragraph{Observed structure.}
Across these debates, thresholded densities range from \(d\approx 0.57\) to \(0.65\), while
mean off-diagonal attack strengths range from \(\mu\approx 0.26\) to \(0.37\). Thus, although all four debates yield moderately dense attack graphs, they differ in how attack mass is distributed. Some debates concentrate high-confidence attacks on a small set of targets, while others spread attacks more evenly across arguments; they also differ in the degree of reciprocal versus asymmetric attack structure.

\paragraph{Interpretation.}
These visualizations are illustrative rather than evidence that any particular topology is
preferable. Their purpose is to show that local pairwise interaction judgments give rise to
distinct global graph geometries, which \textsc{GRASP} subsequently aggregates into stable
rankings. This qualitative variation complements the quantitative analyses linking graph
structure to convergence behavior and ranking consistency.

\subsection{Ablation: Structural Similarity vs.\ Ranking Similarity}
\label{sec:w_vs_rank_similarity}

We study whether similarity between induced attack graphs translates into similarity of
final \textsc{GRASP} rankings. For each unordered pair of judge models \((m_i,m_j)\), and
for each setting (pool or multi-turn), we compute:
\begin{itemize}
    \item the mean Pearson correlation between their attack matrices,
    \(\rho(W^{(i)}, W^{(j)})\), using vectorized off-diagonal entries;
    \item the mean Kendall \(\tau\) correlation between the corresponding final
    \textsc{GRASP} rankings, averaged across debates.
\end{itemize}

Each point in Figure~\ref{fig:w_vs_rank} corresponds to one unordered pair of judge models.
The horizontal axis measures similarity of the induced attack graphs, while the vertical axis
measures similarity of the resulting \textsc{GRASP} rankings.

\begin{figure}[h!]
\centering
\includegraphics[width=\linewidth]{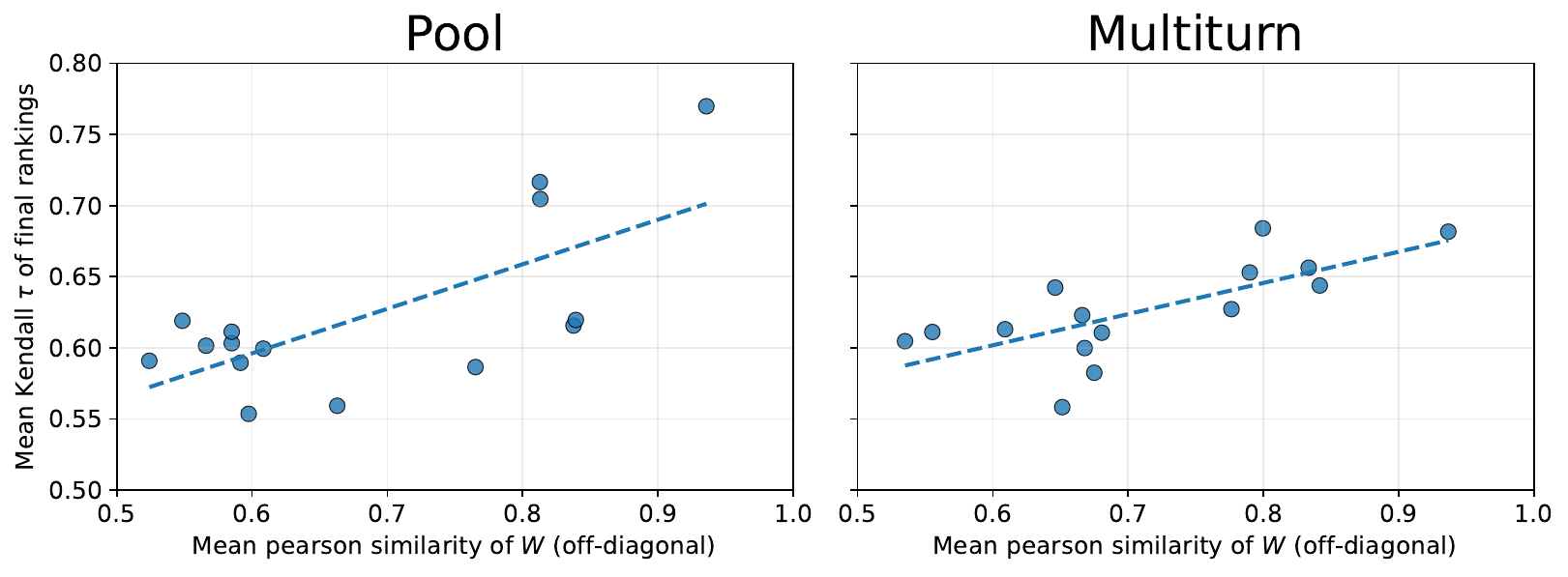}
\caption{Relationship between similarity of induced attack graphs \(W\) and similarity of
final \textsc{GRASP} rankings. Left: pool setting. Right: multi-turn setting. Dashed lines
show least-squares linear fits.}
\label{fig:w_vs_rank}
\end{figure}

Across both settings, we observe a positive association: pairs of judges whose attack
matrices are more strongly correlated tend to produce more similar \textsc{GRASP} rankings.
The trend is consistent in both pool and multi-turn regimes, though with non-negligible
dispersion. This indicates that graph similarity explains a meaningful portion of
inter-model agreement, while the nonlinear propagation dynamics also affect the final
rankings.

\paragraph{Takeaway.}
Agreement at the level of local pairwise attack judgments tends to propagate to agreement in
global \textsc{GRASP} rankings. This supports the view that \textsc{GRASP}'s stability arises
from shared structure in the induced attack graphs rather than from properties of any single
judge model.

\subsection{Angle-Level Agreement Analysis}
\label{app:angle_agreement}

\begin{wraptable}{r}{0.58\textwidth}
\centering
\small
\vspace{-4mm}
\setlength{\tabcolsep}{3pt}
\renewcommand{\arraystretch}{0.92}
\resizebox{0.57\textwidth}{!}{%
\begin{tabular}{l cc cc cc cc}
\toprule
& \multicolumn{2}{c}{Top-3 \(\uparrow\)} 
& \multicolumn{2}{c}{Top-5 \(\uparrow\)}
& \multicolumn{2}{c}{Swap \(\downarrow\)} 
& \multicolumn{2}{c}{\(\rho\) \(\uparrow\)} \\
\cmidrule(lr){2-3} \cmidrule(lr){4-5} \cmidrule(lr){6-7} \cmidrule(lr){8-9}
Angle & G & R & G & R & G & R & G & R \\
\midrule
Social        & \textbf{.888} & .816 & \textbf{.984} & .980 & \textbf{.166} & .309 & \textbf{.746} & .421 \\
Legal         & \textbf{.885} & .794 & \textbf{.982} & .978 & \textbf{.180} & .342 & \textbf{.724} & .352 \\
Moral         & \textbf{.862} & .801 & \textbf{.977} & .963 & \textbf{.175} & .310 & \textbf{.738} & .419 \\
Economic      & \textbf{.857} & .795 & \textbf{.976} & .958 & \textbf{.187} & .320 & \textbf{.705} & .399 \\
Political     & \textbf{.865} & .665 & \textbf{.908} & .861 & \textbf{.162} & .297 & \textbf{.764} & .469 \\
Technological & \textbf{.844} & .632 & \textbf{.904} & .848 & \textbf{.171} & .346 & \textbf{.758} & .359 \\
\bottomrule
\end{tabular}
}
\caption{Angle-level inter-model agreement. G denotes default \textsc{GRASP}; R denotes \textsc{RAW}.}
\label{tab:angle_agreement}
\vspace{-4mm}
\end{wraptable}

We further analyze inter-model agreement at the level of semantic angles
(\textsc{Economic}, \textsc{Legal}, \textsc{Moral}, \textsc{Political},
\textsc{Social}, and \textsc{Technological}). For each debate and angle, we restrict each
model's ranking to the subset of arguments belonging to that angle and compute pairwise
agreement between all model pairs. Metrics are averaged within each debate and then
aggregated across pool and multi-turn settings.

Table~\ref{tab:angle_agreement} compares default \textsc{GRASP} against \textsc{RAW}
rankings. Across all angles, \textsc{GRASP} exhibits substantially higher agreement than
\textsc{RAW}. This indicates that the agreement gains are not driven by a single semantic
dimension, but persist across diverse argumentative frames.

\subsection{External Sanity Check: Recovering Human-Written Point--Counterpoint Relations}
\label{app:idebate_sanity}

The main experiments construct the attack matrix \(W\) using LLM-based directed attack
scoring. A natural concern is whether these models are reliable constructors of local
interaction graphs outside our synthetic debate corpus. We therefore conduct a small
external sanity check on human-written arguments from the iDebate/IDEA Debatabase.\footnote{
\url{https://idebate.net/}
}

\paragraph{Data structure.}
iDebate pages have an editorial point--counterpoint format. Each debate page contains a
motion, followed by two sections: ``Points For'' and ``Points Against.'' Within each section,
a \textsc{POINT} presents an argument for that section, and a paired \textsc{COUNTERPOINT}
provides a direct objection to that specific point. We treat each \textsc{POINT} and each
\textsc{COUNTERPOINT} as a separate argument node, and use the human-authored pairing as an
external signal of directed attack:
\[
\textsc{COUNTERPOINT} \longrightarrow \textsc{POINT}.
\]

The resulting structure is illustrated below:
\begin{center}
\small
\begin{tabular}{p{0.18\linewidth}p{0.34\linewidth}p{0.36\linewidth}}
\toprule
Section & \textsc{POINT} & Paired \textsc{COUNTERPOINT} \\
\midrule
Points For
& Argument supporting the motion
& Objection to that pro argument:
  \(\textsc{COUNTERPOINT}\rightarrow\textsc{POINT}\) \\
\addlinespace[1mm]
Points Against
& Argument opposing the motion
& Objection to that con argument:
  \(\textsc{COUNTERPOINT}\rightarrow\textsc{POINT}\) \\
\bottomrule
\end{tabular}
\end{center}

For stance metadata, this means that \textsc{POINT}s under ``Points For'' are Pro,
\textsc{COUNTERPOINT}s under ``Points For'' are Con, \textsc{POINT}s under ``Points
Against'' are Con, and \textsc{COUNTERPOINT}s under ``Points Against'' are Pro. However,
the diagnostic itself does not depend on stance labels; it only tests whether the induced
attack scores recover the explicit human-authored \(\textsc{COUNTERPOINT}\rightarrow
\textsc{POINT}\) relations.

\paragraph{Example.}
For a page on whether states should ratify the U.N. Convention on the Rights of Migrant
Workers,\footnote{\url{https://idebate.net/that-all-states-should-immediately-ratify-the-u-n-convention-on-the-rights-of-migrant-workers-and-t~b773/}}
one \textsc{POINT} argues that migrants face a growing human-rights problem that requires
international protection. Its paired \textsc{COUNTERPOINT} argues that migration is a broader
policy problem, not primarily a migrant-rights problem, and therefore directly challenges
the framing of the point. In our diagnostic, this pair contributes one labeled directed
attack edge from the \textsc{COUNTERPOINT} to the paired \textsc{POINT}.

\paragraph{Protocol.}
We use 12 iDebate pages, retaining pages with at least 10 extracted argument units. This
yields 304 argument nodes and 152 explicit point--counterpoint pairs. We use the same judge model set as in Section~\ref{sec:baseline_models}, except that
\texttt{meta-llama/llama-4-scout} is replaced with
\texttt{qwen/qwen3.5-flash-02-23} due to provider availability during the external
sanity-check run. For each judge model,
we construct \(W\) using the same pairwise attack-scoring prompt as in the main experiments.
We then compare the score assigned to each explicit
\(\textsc{COUNTERPOINT}\rightarrow\textsc{POINT}\) edge against other incoming edges to the
same \textsc{POINT}. We report: (i) the mean attack score for explicit counterpoint edges,
(ii) the mean attack score for other incoming edges, (iii) the AUC for distinguishing
explicit counterpoint edges from other incoming edges using \(W_{ij}\), and (iv) Hit@\(k\),
the fraction of \textsc{POINT}s for which the paired \textsc{COUNTERPOINT} appears among
the top-\(k\) strongest incoming attackers.

\begin{table}[t]
\centering
\small
\setlength{\tabcolsep}{5pt}
\begin{tabular}{lcccccc}
\toprule
Source
& \(W_{\mathrm{CP}\rightarrow\mathrm{Point}}\)
& \(W_{\mathrm{Other}\rightarrow\mathrm{Point}}\)
& AUC
& Hit@1
& Hit@3
& Hit@5 \\
\midrule
iDebate
& 0.756
& 0.382
& 0.829
& 0.231
& 0.493
& 0.674 \\
\bottomrule
\end{tabular}
\caption{External sanity check on human-written iDebate/IDEA point--counterpoint pairs.
Each explicit \(\textsc{COUNTERPOINT}\rightarrow\textsc{POINT}\) relation is treated as a
human-authored directed attack. Scores are macro-averaged across six LLM judge models.}
\label{tab:idebate_sanity}
\vspace{-7mm}
\end{table}

\paragraph{Results.}
Table~\ref{tab:idebate_sanity} shows that the induced attack scores recover the
human-authored point--counterpoint structure substantially better than chance. Across six
judge models, explicit counterpoint edges receive much higher attack scores than other
incoming edges to the same points (\(0.756\) vs.\ \(0.382\)). The resulting AUC is \(0.829\),
indicating that \(W\) reliably distinguishes human-authored counterpoint attacks from
non-paired incoming edges. The paired counterpoint is also ranked among the top five
incoming attackers for \(67.4\%\) of points. The effect is consistent across all six judge
models: every model assigns a higher mean score to explicit
\(\textsc{COUNTERPOINT}\rightarrow\textsc{POINT}\) edges than to other incoming edges.

This diagnostic supports the use of LLMs as local interaction scorers: although direct
LLM rankings are unstable, their pairwise attack judgments recover externally provided
point--counterpoint relations in human-written debate material. We do not treat iDebate as
a global ranking benchmark, since its editorial structure consists of paired claims and
rebuttals rather than a naturally instantiated multi-argument interaction graph. Instead,
the result isolates the reliability of the local \(W\)-construction step used by
\textsc{GRASP}.

\section{\textsc{StructDebate} Construction Details}
\label{app:dataset-details}

\paragraph{Prompting and controls.}
All generators receive the same prompt template for a given setting, with only the motion,
stance, semantic angle, and debate history varied. In the pool setting, arguments are
generated independently. In the multi-turn setting, each turn receives the previous turns as
context and is instructed to respond from the alternating \textsc{Pro}/\textsc{Con} side.
We use fixed generation templates across models to reduce prompt-induced variation and keep
stance, angle, and debate setting comparable. The exact generation prompt schemas are provided
below.

\paragraph{Why six semantic angles?}
The six angles---\textsc{Economic}, \textsc{Legal}, \textsc{Moral}, \textsc{Political},
\textsc{Social}, and \textsc{Technological}---are controlled prompting dimensions rather
than an exhaustive taxonomy of argument types. They were chosen to induce diverse but
comparable argumentative frames commonly appearing in public-policy debates, while keeping
the generation process simple, balanced, and reproducible. This design follows the use of
qualitative category construction in social computing and applied qualitative research, where
categories are often introduced to support coverage and comparison rather than to define a
complete ontology~\citep{braun2006using,mcdonald2019reliability,glaser2017discovery}.

Concretely, the angles serve three experimental purposes. First, they reduce the chance that
arguments for the same motion collapse into near-paraphrases. Second, they ensure that both
sides of each motion are represented across comparable thematic dimensions. Third, they allow
angle-level robustness checks: if \textsc{GRASP}'s stability were driven by one narrow kind of
argument, this would appear as angle-specific variation. We report such checks in
Appendix~\ref{app:angle_agreement}.

\subsection{Argument Generation Prompt Schemas}
\label{app:generation_prompts}

Below we provide the prompt schemas used to generate \textsc{StructDebate}. All generations
were sampled with default temperature.

\paragraph{POOL argument generation prompt.}
\label{prompt:pool_gen}

\begin{PromptVerbatim}
POOL_SYSTEM = "You generate debate arguments. Output must be valid JSON only."

def pool_user_content(motion, side, angle, k):
    return json.dumps({
        "task": "Generate short debate arguments.",
        "motion": motion,
        "side": side,
        "angle": angle,
        "num_arguments": k,
        "constraints": {
            "length": "2-3 sentences each",
            "style": "plain, analytical, no rhetorical flourish",
            "no_lists": True,
            "no_citations": True,
            "no_quotes": True,
            "one_core_claim_plus_one_reason": True,
            "avoid_metaphor": True
        },
        "output_requirements": [
            "Return ONLY a JSON object.",
            "No markdown, no code fences, no commentary.",
            "Schema: {\"arguments\": [\"...\", \"...\"]}"
        ]
    }, ensure_ascii=False)
\end{PromptVerbatim}

\paragraph{MULTITURN argument generation prompt.}
\label{prompt:multiturn_gen}

\begin{PromptVerbatim}
MULTITURN_SYSTEM = "You participate in a structured debate. Output must be valid JSON only."

def multiturn_user_content(motion, side, angle, history, turn_idx):
    return json.dumps({
        "task": "Write one debate turn.",
        "motion": motion,
        "side": side,
        "required_angle": angle,
        "turn_index": turn_idx,
        "debate_history": history,
        "constraints": {
            "length": "2-4 sentences",
            "must_address_previous": (side == "Con"),
            "style": "plain, analytical, no rhetorical flourish",
            "no_lists": True,
            "no_citations": True,
            "no_quotes": True,
            "avoid_metaphor": True
        },
        "output_requirements": [
            "Return ONLY a JSON object.",
            "No markdown, no code fences, no commentary.",
            "Schema: {\"text\": \"...\"}"
        ]
    }, ensure_ascii=False)
\end{PromptVerbatim}

\subsection{Debate Motions}
\label{app:motions}

We list the 50 debate motions used in \textsc{StructDebate}, grouped by thematic category.

\paragraph{Technology \& AI}
\begin{enumerate}
\setcounter{enumi}{0}
\item This House would ban the use of AI in primary and secondary education.
\item This House would ban stablecoins pegged to national currencies.
\item This House would mandate all businesses to accept only digital payments.
\item This House would require electric vehicle manufacturers to refuse sales in countries with poor environmental records.
\item This House would allow individuals to erase morally distressing memories.
\item This House would ban facial recognition technology in public spaces.
\item This House would require social media companies to make their recommendation algorithms public.
\end{enumerate}

\paragraph{Economics \& Labor}
\begin{enumerate}
\setcounter{enumi}{7}
\item This House would abolish the minimum wage law.
\item This House would allow the sale and purchase of human organs.
\item This House would ban sovereign wealth funds from investing in private equity.
\item This House would require companies to make the salaries of all their employees publicly available.
\item This House would allow workers in less economically developed countries to waive labor protections in exchange for higher wages.
\item This House would allow lump-sum scholarships as an alternative to periodic disbursements.
\item This House would introduce a universal basic income funded by wealth taxes.
\end{enumerate}

\paragraph{Law, Rights, and Governance}
\begin{enumerate}
\setcounter{enumi}{14}
\item This House would require warrants for searches instead of allowing stop-and-frisk.
\item This House would criminalize dangerous in-play actions in professional sport.
\item This House would ban corporate donations to political campaigns.
\item This House would allow democratic governments to overturn supranational court decisions with a simple legislative majority.
\item This House would introduce a binding ``None of the Above'' option on national election ballots.
\item This House would allow citizens to vote directly on impeachment cases through a national referendum.
\item This House would allow subnational jurisdictions to overturn federal policies via citizen referendum.
\end{enumerate}

\paragraph{Social \& Moral Policy}
\begin{enumerate}
\setcounter{enumi}{21}
\item This House would allow parents to administer behavioral enhancement drugs to children without their consent.
\item This House would mandate comprehensive queer education in schools.
\item This House would restrict state funding only to art perceived as valuable by the general public.
\item This House would ban the private ownership of historical artifacts.
\item This House would allow high school students to rate teachers as a primary basis for pay increases.
\item This House would introduce compulsory national service for all citizens.
\end{enumerate}

\paragraph{Environment \& Development}
\begin{enumerate}
\setcounter{enumi}{27}
\item This House would invest preferentially in climate startups in developing countries rather than developed countries.
\item This House would ban the export of waste to developing countries.
\item This House would nationalize luxury ingredient production in producing countries.
\item This House would impose carbon tariffs on imported goods.
\item This House would ban advertising for environmentally harmful products.
\end{enumerate}

\paragraph{International Relations \& Identity}
\begin{enumerate}
\setcounter{enumi}{32}
\item This House would ban countries from offering financial incentives to foreign athletes to switch nationality.
\item This House would allow prisoners serving life without parole to opt for the death penalty.
\item This House would ban proselytization acts in liberal democracies.
\item This House would restrict immigration based on environmental carrying capacity.
\item This House would allow refugees to be settled through private sponsorship markets.
\end{enumerate}

\paragraph{Media, Culture, and Education}
\begin{enumerate}
\setcounter{enumi}{37}
\item This House would ban the consolidation of major news organizations.
\item This House would introduce a youth-weighted voting system in democratic elections.
\item This House would require public broadcasters to allocate equal airtime to all political parties.
\item This House would abolish standardized testing in university admissions.
\item This House would mandate media literacy education for all adults.
\end{enumerate}

\paragraph{Security \& State Power}
\begin{enumerate}
\setcounter{enumi}{42}
\item This House would ban private military contractors.
\item This House would allow preventive detention for credible terrorism threats.
\item This House would restrict police use of lethal force to extreme circumstances only.
\end{enumerate}

\paragraph{Science \& Bioethics}
\begin{enumerate}
\setcounter{enumi}{45}
\item This House would allow gene editing of embryos for non-medical traits.
\item This House would require mandatory vaccination for all citizens.
\item This House would ban animal testing for cosmetic products.
\end{enumerate}

\paragraph{Platform Power \& Markets}
\begin{enumerate}
\setcounter{enumi}{48}
\item This House would break up dominant technology monopolies.
\item This House would require platforms to compensate users for personal data usage.
\end{enumerate}

\section{Hyperparameter Sensitivity via Cross-Model Agreement}
\label{app:gridsearch}

We perform a \emph{post-hoc} grid search to analyze the sensitivity of \textsc{GRASP} to
its hyperparameters, using cross-model agreement among \textsc{GRASP}-induced rankings as
the diagnostic signal. If \textsc{GRASP} captures a stable structural signal, then attack
graphs induced by different judge models should yield similar final rankings across a broad
range of hyperparameter choices. This analysis is diagnostic only: the main experiments use
a single \emph{a priori} hyperparameter setting, and the grid search is used to assess
robustness rather than to tune the reported results.

\paragraph{Protocol.}
For each triple \((\alpha,\beta,\gamma)\), we compute a \textsc{GRASP} ranking for each
judge model and measure mean pairwise Kendall-\(\tau\) agreement across all judge pairs. We
then average this agreement across debates. We search over
\[
\alpha \in \{0.1, 0.25, 0.5, 1.0\}, \qquad
\beta  \in \{0.1, 0.25, 0.5, 0.75\}, \qquad
\gamma \in \{0.6, 0.8, 0.9, 1.0\}.
\]

\paragraph{Results.}
The best post-hoc configuration is
\[
\alpha = 1.0, \qquad \beta = 0.25, \qquad \gamma = 0.6,
\]
achieving mean pairwise Kendall-\(\tau = 0.624\) across 300 debates.The main-paper setting \((\alpha,\beta,\gamma)=(1.0,0.6,0.9)\) obtains Kendall-\(\tau = 0.6245\), within \(0.0005\) of the post-hoc optimum. This indicates that the agreement gains are not sensitive to a narrow hyperparameter choice.

\section{Judging Prompts and Prompt Optimization}
\label{app:prompts}

\subsection{Prompt Optimization for RAW Rankings}
\label{app:raw_prompt_opt}

We test whether refining the \textsc{RAW} ranking prompt improves inter-judge agreement.
All prompts are evaluated with temperature \(0\). We compare three variants: the original
\textsc{RAW} prompt used in the main paper, an adjudication-style prompt emphasizing logic,
impact, and relevance (\textsc{RAW-v2}), and a comparative prompt emphasizing supersession
and logical dominance (\textsc{RAW-v3}).

\begin{wraptable}{r}{0.62\textwidth}
\centering
\small
\vspace{-4mm}
\setlength{\tabcolsep}{3pt}
\renewcommand{\arraystretch}{0.92}
\resizebox{0.61\textwidth}{!}{%
\begin{tabular}{l cc cc cc cc cc}
\toprule
& \multicolumn{2}{c}{\(\tau \uparrow\)}
& \multicolumn{2}{c}{Swap \(\downarrow\)}
& \multicolumn{2}{c}{\(\rho \uparrow\)}
& \multicolumn{2}{c}{Top-3 \(\uparrow\)}
& \multicolumn{2}{c}{Top-5 \(\uparrow\)} \\
\cmidrule(lr){2-3} \cmidrule(lr){4-5} \cmidrule(lr){6-7}
\cmidrule(lr){8-9} \cmidrule(lr){10-11}
Prompt & P & MT & P & MT & P & MT & P & MT & P & MT \\
\midrule
Original & \textbf{.34} & \textbf{.31} & \textbf{.33} & \textbf{.35}
         & \textbf{.43} & \textbf{.38} & .39 & .41 & .42 & .47 \\
RAW-v2   & .18 & .15 & .41 & .43 & .23 & .19
         & .40 & \textbf{.51} & .46 & \textbf{.57} \\
RAW-v3   & .17 & .11 & .41 & .44 & .24 & .16
         & \textbf{.41} & .41 & \textbf{.48} & .48 \\
\bottomrule
\end{tabular}
}
\caption{Effect of \textsc{RAW} prompt optimization on inter-model agreement. P and MT denote
Pool and Multi-turn settings.}
\label{tab:raw_prompt_opt}
\vspace{-4mm}
\end{wraptable}

Table~\ref{tab:raw_prompt_opt} reports inter-model agreement for each prompt. Neither
refined prompt improves Kendall agreement or Spearman correlation over the original
\textsc{RAW} prompt. Both variants substantially reduce Kendall's \(\tau\), indicating
weaker global consistency across judges. Although small gains appear in Top-\(k\) overlap,
they coincide with increased swap distance, suggesting that limited agreement on a few top
arguments masks broader ranking instability.

\paragraph{Takeaway.}
Prompt engineering alone does not make end-to-end \textsc{RAW} ranking reliable. Even more
structured adjudication instructions fail to recover high inter-model agreement, reinforcing
the need for structural aggregation from local pairwise interactions.

\subsection{Judging Prompt Schemas}
\label{app:prompt_schemas}

Below we provide the prompt schemas used for direct ranking and pairwise attack scoring.

\paragraph{RAW ranking prompt (Original).}
\label{prompt:original}

\begin{PromptVerbatim}
RAW_SYSTEM = (
    "You are a careful debate judge. "
    "Rank arguments by how strong and sufficient they are. "
    "Return ONLY valid JSON."
)

def raw_user_payload(motion: str, args: list[dict]):
    return {
        "task": "Rank debate arguments by structural strength for the motion.",
        "motion": motion,
        "arguments": [
            {
                "id": a["arg_id"],
                "side": a["side"],
                "angle": a["angle"],
                "turn": int(a["turn"]),
                "text": a["text_trunc"],
            }
            for a in args
        ],
        "output_requirements": [
            "Return ONLY a JSON object.",
            "No markdown, no code fences, no commentary.",
            "Schema: {\"ranking\": [\"<arg_id>\", ...]}",
            "ranking must contain each input id exactly once."
        ],
    }
\end{PromptVerbatim}

\paragraph{RAW ranking prompt (v2).}
\label{prompt:v2}

\begin{PromptVerbatim}
RAW_SYSTEM = (
    "You are an expert World Schools Debate adjudicator. "
    "Your goal is to evaluate arguments based on logical coherence, evidence, and impact. "
    "You must remain neutral and ignore your own stance on the motion. "
    "Return ONLY valid JSON."
)

def raw_user_payload(motion: str, args: list[dict]):
    return {
        "task": "Rank the provided debate arguments from strongest to weakest.",
        "motion": motion,
        "evaluation_criteria": {
            "1. Logic": "Are the premises true and does the conclusion follow? Is the reasoning explained clearly?",
            "2. Impact": "Does the argument show why this outcome matters significantly to the stakeholders?",
            "3. Relevance": "How directly does it address the specific motion provided?"
        },
        "arguments": [
            {
                "id": a["arg_id"],
                "side": a["side"],
                "text": a["text_trunc"],
            }
            for a in args
        ],
        "output_requirements": [
            "Return ONLY a raw JSON object.",
            "DO NOT include markdown formatting.",
            "Schema: {\"ranking\": [\"<arg_id_best>\", ...]}",
            "ranking must contain every input id exactly once."
        ],
    }
\end{PromptVerbatim}

\paragraph{RAW ranking prompt (v3).}
\label{prompt:v3}

\begin{PromptVerbatim}
RAW_SYSTEM = (
    "You are a rigorous logic engine designed to compare debating points. "
    "Determine which arguments successfully supersede or outweigh the others. "
    "Return ONLY valid JSON."
)

def raw_user_payload(motion: str, args: list[dict]):
    return {
        "task": "Perform a comparative ranking of the debate arguments for the given motion.",
        "motion": motion,
        "instructions": [
            "Read all arguments first.",
            "Identify arguments that rely on logical fallacies and rank them lower.",
            "Identify arguments with strong mechanisms and high-stakes impacts and rank them higher.",
            "If two arguments are similar, rank the one with more nuance/detail higher."
        ],
        "arguments": [
            {
                "id": a["arg_id"],
                "side": a["side"],
                "angle": a["angle"],
                "text": a["text_trunc"],
            }
            for a in args
        ],
        "output_requirements": [
            "Output valid JSON only.",
            "No prologue or epilogue.",
            "Schema: {\"ranking\": [\"<strongest_arg_id>\", ..., \"<weakest_arg_id>\"]}",
            "Ensure strict adherence to the schema."
        ],
    }
\end{PromptVerbatim}

\paragraph{Pairwise attack-scoring prompt used to construct \(W\) for GRASP.}
\label{prompt:nli}

\begin{PromptVerbatim}
NLI_SYSTEM = (
    "You are an interaction scorer for debate arguments.\n"
    "Given Argument A and Argument B, output how strongly Argument A attacks "
    "or undermines Argument B.\n\n"
    "Return ONLY valid JSON with keys:\n"
    "{\"attack_score\": number}\n\n"
    "Rules:\n"
    "- attack_score must be a real-valued number in [0,1].\n"
    "- 0.0 means: Argument A does not undermine Argument B at all "
    "(supportive or unrelated).\n"
    "- 1.0 means: Argument A directly contradicts or strongly refutes Argument B.\n"
    "- Use the full continuous range; do NOT restrict to discrete steps.\n"
    "- You must always output a score (never null).\n"
    "- Output JSON only. No extra text."
)

def nli_user_payload(attacker_text: str, target_text: str):
    return {
        "task": "Directed attack scoring",
        "argument_a_attacker": attacker_text,
        "argument_b_target": target_text,
        "output_format": {"attack_score": "float in [0,1]"},
        "rules": ["Output JSON only. No markdown. No extra keys."],
    }
\end{PromptVerbatim}

\paragraph{RAW+SS ranking prompt.}
\label{prompt:raw_ss}

\begin{PromptVerbatim}
RAW_SS_SYSTEM = (
    "You are a careful structural argument judge. "
    "Rank arguments by structural sufficiency, not by persuasion, truth, rhetoric, style, factuality, "
    "or your own agreement with the motion. "
    "Return ONLY valid JSON."
)

STRUCTURAL_SUFFICIENCY_TEXT = """
Structural sufficiency is a graph-relative notion of argument robustness.

An argument is structurally strong when it withstands the explicit attacks present in the debate.
An argument should be ranked higher if its attackers are themselves countered, weakened, or answered by other arguments.

Do not rank by:
- truth,
- factual correctness,
- persuasiveness,
- rhetorical polish,
- writing style,
- verbosity,
- your own agreement with the claim.

Focus only on how well the argument is positioned in the explicit attack-defense structure of the debate.
"""

def raw_ss_user_payload(motion: str, args: list[dict]):
    return {
        "task": "Rank debate arguments by structural sufficiency.",
        "motion": motion,
        "definition": STRUCTURAL_SUFFICIENCY_TEXT,
        "arguments": [
            {
                "id": a["arg_id"],
                "side": a["side"],
                "angle": a["angle"],
                "turn": int(a["turn"]),
                "text": a["text_trunc"],
            }
            for a in args
        ],
        "output_requirements": [
            "Return ONLY a JSON object.",
            "No markdown, no code fences, no commentary.",
            "Schema: {\"ranking\": [\"<arg_id>\", ...]}",
            "ranking must contain each input id exactly once."
        ],
    }
\end{PromptVerbatim}

\section{Case Study Argument Texts}
\label{app:case_texts}

\begingroup
\small
\renewcommand{\arraystretch}{1.3} 
\begin{xltabular}{\textwidth}{l l l r r X}
\caption{Metadata and full text for the most rank-volatile arguments in debate MT-048. \# Attk. and Mean $\mu$ are obtained via $W$ (\texttt{gpt-5.2-chat}).} \label{tab:case_texts} \\

\toprule
Arg & Stance & Angle & \#Attk. & Mean $\mu$ & Text \\
\midrule
\endfirsthead

\multicolumn{6}{c}{{\bfseries \tablename\ \thetable{} -- Continued from previous page}} \\
\toprule
Arg & Stance & Angle & \#Attk. & Mean $\mu$ & Text \\
\midrule
\endhead

\midrule
\multicolumn{6}{r}{{Continued on next page...}} \\
\endfoot

\bottomrule
\endlastfoot

arg\_19 & Pro & Social & 10 & 0.6850 &
Dominant technology monopolies exacerbate social inequalities by prioritizing content that favors affluent users and marginalizing voices from lower socioeconomic groups. Breaking them up would enable smaller platforms to cater specifically to diverse demographics, fostering more inclusive online communities and reducing digital divides. This change would also improve social cohesion by decentralizing control over algorithms that currently amplify polarizing content. \\ \addlinespace

arg\_2 & Con & Techn. & 12 & 0.6400 &
Dominant technology monopolies drive technological innovation by concentrating resources for large-scale research and development that smaller entities cannot match, countering the claim that they suppress advancements. Breaking them up would fragment essential platforms and standards, potentially slowing the integration of new technologies across ecosystems. This fragmentation could limit rather than promote broader access to cutting-edge solutions, as coordinated efforts by monopolies often accelerate widespread adoption. \\ \addlinespace

arg\_4 & Con & Pol. & 12 & 0.5683 &
Dominant technology monopolies provide a centralized point for political accountability, allowing governments to engage with fewer entities for effective oversight rather than dealing with fragmented influences from multiple smaller companies. Breaking them up would likely increase overall lobbying efforts as numerous firms compete for policy favors, potentially overwhelming democratic processes instead of diluting corporate power. This concentration enables more efficient implementation of regulations that address public welfare without the chaos of dispersed political pressures. \\ \addlinespace

arg\_15 & Pro & Econ. & 10 & 0.6700 &
Dominant technology monopolies distort economic efficiency by capturing excessive profits through barriers to entry that prevent efficient resource allocation across industries. Breaking them up would allow for more dynamic markets where new entrants can compete on merit, leading to improved productivity and broader economic distribution of wealth. This change would also reduce the risk of market failures associated with over-reliance on a few firms for critical technological services. \\ \addlinespace

arg\_17 & Pro & Moral & 10 & 0.6970 &
Dominant technology monopolies create moral issues by enabling the exploitation of user data without sufficient accountability, which undermines trust in digital systems and harms individual dignity. Breaking them up would foster a more ethical environment where multiple companies must compete on the basis of responsible practices rather than relying on unchecked dominance. This division would also reduce the moral risks associated with concentrated control over information flows that can amplify societal divisions. \\ \addlinespace

arg\_11 & Pro & Social & 10 & 0.5990 &
Dominant technology monopolies contribute to social isolation by designing platforms that prioritize addictive engagement over meaningful interactions among users. Breaking them up would encourage the development of diverse social networks that facilitate healthier community building and reduce echo chambers. This restructuring would also enhance social equity by allowing smaller entities to address the needs of underserved populations more effectively. \\

\end{xltabular}
\endgroup

\section{Additional Qualitative Case Studies}
\label{app:contrast_cases}

We present four additional cases exhibiting the same pattern:
near-unanimous high ranking under \textsc{GRASP} and highly dispersed, low ranking under \textsc{RAW}. In all the examples, \# Attackers and Mean attack strength are obtained from the $W$ constructed by the \texttt{openai/gpt-5.2-chat} judge. 

\subsection*{Case A (Multi-turn)}

\begin{tcolorbox}[colback=green!5, colframe=green!60!black, title=\textbf{This House would introduce a binding ``None of the Above'' option on national election ballots.}, fonttitle=\bfseries, bottom=1mm]
\small
\begin{tabular*}{\linewidth}{@{\extracolsep{\fill}}ll ll ll@{}}
\textbf{Generator:} & \texttt{gpt-5.2-pro} & \textbf{Stance:} & Pro & \textbf{Angle:} & Tech. \\
\textbf{Turn:} & 2 & \textbf{\# Attk:} & 10 & \textbf{Mean:} & 0.205
\end{tabular*}
\vspace{1mm} \hrule \vspace{2mm}
\textbf{Argument:} \emph{A binding None of the Above can be implemented cleanly in modern election technology because it is just another selectable contest option with deterministic tabulation and a predefined trigger for reruns, reducing reliance on informal protest mechanisms like ballot spoilage that are harder to interpret and audit. It improves data quality by distinguishing abstention, undervotes, and explicit rejection, which enables better diagnostics of candidate quality and voter dissatisfaction without compromising ballot secrecy. In jurisdictions using paper-backed electronic systems, NOTA can be captured on both the human-readable ballot and the cast-vote record, supporting risk-limiting audits and minimizing disputes about intent.}
\end{tcolorbox}

\noindent
\begin{minipage}[t]{0.49\textwidth}
    \begin{tcolorbox}[colback=blue!5, colframe=blue!70!black, title=\textbf{GRASP Rankings}, fonttitle=\bfseries, bottom=1mm]
    \footnotesize
    \texttt{anthropic/claude-haiku-4.5}: 1 \\
    \texttt{deepseek/deepseek-v3.2}: 1 \\
    \texttt{google/gemini-3-flash-preview}: 1 \\
    \texttt{meta-llama/llama-4-scout}: 3 \\
    \texttt{openai/gpt-5.2-chat}: 1 \\
    \texttt{xiaomi/mimo-v2-flash}: 3
    \end{tcolorbox}
\end{minipage}\hfill
\begin{minipage}[t]{0.49\textwidth}
    \begin{tcolorbox}[colback=red!5, colframe=red!70!black, title=\textbf{RAW Rankings}, fonttitle=\bfseries, bottom=1mm]
    \footnotesize
    \texttt{anthropic/claude-haiku-4.5}: 19 \\
    \texttt{deepseek/deepseek-v3.2}: 19 \\
    \texttt{google/gemini-3-flash-preview}: 8 \\
    \texttt{meta-llama/llama-4-scout}: 5 \\
    \texttt{openai/gpt-5.2-chat}: 10 \\
    \texttt{xiaomi/mimo-v2-flash}: 20
    \end{tcolorbox}
\end{minipage}

\subsection*{Case B (Pool)}

\begin{tcolorbox}[colback=green!5, colframe=green!60!black, title=\textbf{This House would restrict state funding only to art perceived as valuable by the general public.}, fonttitle=\bfseries, bottom=1mm]
\small
\begin{tabular*}{\linewidth}{@{\extracolsep{\fill}}ll ll ll@{}}
\textbf{Generator:} & \texttt{mistral-sm-c} & \textbf{Stance:} & Pro & \textbf{Angle:} & Social \\
\textbf{Setting:} & Pool & \textbf{\# Attk:} & 29 & \textbf{Mean:} & 0.271
\end{tabular*}
\vspace{1mm} \hrule \vspace{2mm}
\textbf{Argument:} \emph{Public funding for widely valued art maximizes the likelihood of creating accessible cultural experiences, ensuring that marginalized or economically disadvantaged groups can engage with meaningful creative expressions, thereby addressing systemic inequalities in cultural participation.}
\end{tcolorbox}

\noindent
\begin{minipage}[t]{0.49\textwidth}
    \begin{tcolorbox}[colback=blue!5, colframe=blue!70!black, title=\textbf{GRASP Rankings}, fonttitle=\bfseries, bottom=1mm]
    \footnotesize
    \texttt{anthropic/claude-haiku-4.5}: 1 \\
    \texttt{deepseek/deepseek-v3.2}: 3 \\
    \texttt{google/gemini-3-flash-preview}: 2 \\
    \texttt{meta-llama/llama-4-scout}: 1 \\
    \texttt{openai/gpt-5.2-chat}: 2 \\
    \texttt{xiaomi/mimo-v2-flash}: 1
    \end{tcolorbox}
\end{minipage}\hfill
\begin{minipage}[t]{0.49\textwidth}
    \begin{tcolorbox}[colback=red!5, colframe=red!70!black, title=\textbf{RAW Rankings}, fonttitle=\bfseries, bottom=1mm]
    \footnotesize
    \texttt{anthropic/claude-haiku-4.5}: 11 \\
    \texttt{deepseek/deepseek-v3.2}: 8 \\
    \texttt{google/gemini-3-flash-preview}: -- \\
    \texttt{meta-llama/llama-4-scout}: -- \\
    \texttt{openai/gpt-5.2-chat}: 33 \\
    \texttt{xiaomi/mimo-v2-flash}: 33
    \end{tcolorbox}
\end{minipage}

\subsection*{Case C (Pool)}

\begin{tcolorbox}[colback=green!5, colframe=green!60!black, title=\textbf{This House would introduce a youth-weighted voting system in democratic elections.}, fonttitle=\bfseries, bottom=1mm]
\small
\begin{tabular*}{\linewidth}{@{\extracolsep{\fill}}ll ll ll@{}}
\textbf{Generator:} & \texttt{qwen3-max} & \textbf{Stance:} & Pro & \textbf{Angle:} & Tech. \\
\textbf{Setting:} & Pool & \textbf{\# Attk:} & 17 & \textbf{Mean:} & 0.17
\end{tabular*}
\vspace{1mm} \hrule \vspace{2mm}
\textbf{Argument:} \emph{Secure blockchain-based voting platforms can ensure transparency and prevent tampering in a youth-weighted system, maintaining trust while accommodating variable vote weights. These technologies already support complex voting rules in pilot programs.}
\end{tcolorbox}

\noindent
\begin{minipage}[t]{0.49\textwidth}
    \begin{tcolorbox}[colback=blue!5, colframe=blue!70!black, title=\textbf{GRASP Rankings}, fonttitle=\bfseries, bottom=1mm]
    \footnotesize
    \texttt{anthropic/claude-haiku-4.5}: 2 \\
    \texttt{deepseek/deepseek-v3.2}: 1 \\
    \texttt{google/gemini-3-flash-preview}: 1 \\
    \texttt{meta-llama/llama-4-scout}: 3 \\
    \texttt{openai/gpt-5.2-chat}: 1 \\
    \texttt{xiaomi/mimo-v2-flash}: 1
    \end{tcolorbox}
\end{minipage}\hfill
\begin{minipage}[t]{0.49\textwidth}
    \begin{tcolorbox}[colback=red!5, colframe=red!70!black, title=\textbf{RAW Rankings}, fonttitle=\bfseries, bottom=1mm]
    \footnotesize
    \texttt{anthropic/claude-haiku-4.5}: 20 \\
    \texttt{deepseek/deepseek-v3.2}: 20 \\
    \texttt{google/gemini-3-flash-preview}: 38 \\
    \texttt{meta-llama/llama-4-scout}: 17 \\
    \texttt{openai/gpt-5.2-chat}: 38 \\
    \texttt{xiaomi/mimo-v2-flash}: 39
    \end{tcolorbox}
\end{minipage}

\subsection*{Case D (Pool)}

\begin{tcolorbox}[colback=green!5, colframe=green!60!black, title=\textbf{This House would allow gene editing of embryos for non-medical traits.}, fonttitle=\bfseries, bottom=1mm]
\small
\begin{tabular*}{\linewidth}{@{\extracolsep{\fill}}ll ll ll@{}}
\textbf{Generator:} & \texttt{grok-4} & \textbf{Stance:} & Pro & \textbf{Angle:} & Pol. \\
\textbf{Setting:} & Pool & \textbf{\# Attk:} & 21 & \textbf{Mean:} & 0.170
\end{tabular*}
\vspace{1mm} \hrule \vspace{2mm}
\textbf{Argument:} \emph{This approach strengthens national innovation in biotechnology from a political perspective. Governments that permit such editing encourage research and development, positioning the country as a leader in global scientific progress and enhancing its geopolitical influence.}
\end{tcolorbox}

\noindent
\begin{minipage}[t]{0.49\textwidth}
    \begin{tcolorbox}[colback=blue!5, colframe=blue!70!black, title=\textbf{GRASP Rankings}, fonttitle=\bfseries, bottom=1mm]
    \footnotesize
    \texttt{anthropic/claude-haiku-4.5}: 2 \\
    \texttt{deepseek/deepseek-v3.2}: 1 \\
    \texttt{google/gemini-3-flash-preview}: 1 \\
    \texttt{meta-llama/llama-4-scout}: 2 \\
    \texttt{openai/gpt-5.2-chat}: 1 \\
    \texttt{xiaomi/mimo-v2-flash}: 1
    \end{tcolorbox}
\end{minipage}\hfill
\begin{minipage}[t]{0.49\textwidth}
    \begin{tcolorbox}[colback=red!5, colframe=red!70!black, title=\textbf{RAW Rankings}, fonttitle=\bfseries, bottom=1mm]
    \footnotesize
    \texttt{anthropic/claude-haiku-4.5}: 14 \\
    \texttt{deepseek/deepseek-v3.2}: 35 \\
    \texttt{google/gemini-3-flash-preview}: 20 \\
    \texttt{meta-llama/llama-4-scout}: -- \\
    \texttt{openai/gpt-5.2-chat}: 35 \\
    \texttt{xiaomi/mimo-v2-flash}: 38
    \end{tcolorbox}
\end{minipage}

\paragraph{Takeaway.}
Across these additional cases, we observe the same qualitative pattern: arguments that are consistently identified by \textsc{GRASP} as structurally central (highly ranked across all judge models) are simultaneously relegated to low and highly dispersed ranks under \textsc{RAW}. These arguments tend to exhibit moderate to large numbers of attackers with non-trivial mean attack strength, indicating that \textsc{GRASP} promotes arguments whose importance emerges from their global position in the dialectical graph rather than from isolated surface persuasiveness. This further supports the claim that structural aggregation yields more stable and semantically grounded prioritization than direct judge rankings.

\section{Debate Decision Outcomes: Structural Strength vs. Convincingness}
\label{sec:ddo}

A natural question is whether \textsc{GRASP} can serve as a proxy for persuasive success or
rhetorical effectiveness. We test this on the Debate Decision Outcomes (DDO)
dataset~\citep{durmus2018exploring,durmus2019corpus}, which contains multi-round debates
with human votes indicating which side was more convincing, as well as point/status changes
reflecting debate performance. Our goal is not to optimize \textsc{GRASP} for persuasion,
but to test whether a structural-dialectical score implicitly correlates with human
convincingness.

\subsection{Dataset and Filtering}
\label{app:ddo_dataset}

\begin{wraptable}{r}{6.2cm}
\centering
\small
\vspace{-6mm}
\setlength{\tabcolsep}{4pt}
\aboverulesep=0ex
\belowrulesep=0ex
\begin{tabular}{@{}lr@{}}
\toprule
Stage & \# Debates \\
\midrule
Initial dataset & 78{,}376 \\
Fewer than 2 rounds removed & 2{,}484 \\
Fewer than 5 votes removed & 64{,}509 \\
Forfeits removed & 3{,}691 \\
Unknown outcome removed & 189 \\
Ties removed & 536 \\
\midrule
\textbf{Final retained} & \textbf{6{,}967} \\
\bottomrule
\end{tabular}
\caption{DDO filtering stages.}
\label{tab:filtering_stages}
\vspace{-6mm}
\end{wraptable}

Starting from 78{,}376 debates, we remove debates with fewer than two rounds, fewer than
five votes, forfeits, unknown outcomes, fewer than three usable votes, or tied convincingness
labels. The final filtered set contains 6{,}967 debates: 2{,}831 Pro wins and 4{,}136 Con
wins. We use train/validation/test splits of 1{,}000/250/5{,}717 debates, used only for
hyperparameter selection and final evaluation.

\subsection{Experimental Setup}
\label{app:ddo_setup}

For each debate, we construct a weighted attack matrix \(W\) using RoBERTa-large-MNLI. Each
entry \(W_{ij}\) is the contradiction probability for the ordered pair in which argument
\(a_i\) attacks argument \(a_j\). We define \(D=W^2\), compute argument-level \textsc{GRASP}
scores, and obtain side-level scores by summing argument scores for each side. The
higher-scoring side is predicted as the winner.

\begin{wraptable}{r}{6.8cm}
\centering
\small
\vspace{-4mm}
\setlength{\tabcolsep}{4pt}
\aboverulesep=0ex
\belowrulesep=0ex
\begin{tabular}{lcc}
\toprule
Method & Conv. Acc. & Status Acc. \\
\midrule
\rowcolor{gray!15}
\textbf{GRASP}              & 0.493 & 0.494 \\
GRASP-W$_\infty$            & 0.500 & 0.502 \\
GRASP-W$_1$                 & 0.499 & 0.503 \\
GRASP-W$_\infty$+$\bar{D}$  & 0.495 & 0.497 \\
GRASP-W$_1$+$\bar{D}$       & 0.491 & 0.493 \\
\midrule
H-Categorizer               & 0.508 & 0.510 \\
Binary In-Degree            & 0.507 & 0.510 \\
Max In-Degree               & 0.508 & 0.510 \\
Katz Centrality             & \textbf{0.508} & \textbf{0.512} \\
\bottomrule
\end{tabular}
\caption{Performance on DDO. All structural methods operate near chance, indicating that
structural graph strength is not a proxy for human convincingness.}
\label{tab:ddo_results}
\vspace{-9mm}
\end{wraptable}

We tune \textsc{GRASP} hyperparameters on the validation split and select
\(\alpha=0.5\), \(\beta=0.25\), and \(\gamma=1.0\), which yields validation accuracy
\(0.505\) for convincingness and \(0.514\) for point/status prediction. We compare against
the structural baselines defined in Appendix~\ref{subsec:struct_methods}. We report accuracy
for predicting the convincing side and the point/status winner, and Spearman correlation
between \textsc{GRASP} score difference and convincingness margin.

\subsection{Results}
\label{app:ddo_results}

Table~\ref{tab:ddo_results} shows that all structural methods perform near chance on both
convincingness and point/status prediction. \textsc{GRASP} variants are comparable to simple
structural baselines, and none provide meaningful predictive signal for human debate
outcomes. We further compute the Spearman correlation between \textsc{GRASP} side-score
difference \(\Delta s_{\textsc{GRASP}}\) and the convincingness margin, obtaining
\(\rho=-0.009\), indicating no meaningful monotonic association.

These results support the interpretation used throughout the paper: \textsc{GRASP} measures
structural robustness in an explicit interaction graph, not human convincingness,
persuasive effectiveness, or debate outcome likelihood.

\section{GRASP Pseudocode}
\label{app:grasp_pseudocode}

\begin{lstlisting}[style=graspPython,caption={GRASP fixed-point iteration (Python).},label={lst:grasp_pseudocode}]
import numpy as np

def safe_div(a, b, eps=1e-12):
    return a / (b + eps)

def grasp_scores(W, alpha=1.0, beta=0.25, gamma=0.6, max_iters=2000, tol=1e-10):
    #GRASP fixed-point iteration (unnormalized).
    W = np.maximum(W, 0.0).astype(np.float64)
    np.fill_diagonal(W, 0.0)

    # two-step defense: paths of length two
    D = W @ W
    D = np.maximum(D, 0.0)
    np.fill_diagonal(D, 0.0)

    n = W.shape[0]
    s = np.ones(n, dtype=np.float64)

    for _ in range(max_iters):
        atk = W.T @ s          # total incoming attack
        dfn = D.T @ s          # total incoming defense (two-step)

        F = safe_div(1.0 + beta * dfn, 1.0 + alpha * atk)
        s_next = (1.0 - gamma) * s + gamma * F

        if np.max(np.abs(s_next - s)) < tol:
            s = s_next
            break
        s = s_next

    return s

def grasp_ranking(W, **kwargs):
    s = grasp_scores(W, **kwargs)
    order = np.argsort(-s)     # best first
    return order, s
\end{lstlisting}



\end{document}